%% file: main.tex
\documentclass[]{article}
\usepackage{url}            
\usepackage{booktabs}       
\usepackage{amsfonts}       

\usepackage{graphicx}
\usepackage{subcaption}

\usepackage[margin=1in]{geometry}

\usepackage{tcolorbox}
\usepackage[hidelinks]{hyperref}
\definecolor{darkred}{RGB}{150,0,0}
\definecolor{darkgreen}{RGB}{0,150,0}
\definecolor{darkblue}{RGB}{0,0,150}
\hypersetup{colorlinks=true, linkcolor=red, citecolor=darkgreen, urlcolor=darkblue}

\usepackage{tikz}
\usepackage{amssymb}
\usepackage{amsmath}
\usepackage{amsmath,amsfonts,amssymb,amsthm}
\usepackage{mathtools}
\usepackage{commath}
\usepackage{flexisym}
\usepackage[utf8]{inputenc}
\usepackage{csquotes}
\usepackage[english]{babel}

\usepackage{color}
\usepackage{bbm}

\usepackage{amsmath}
\DeclareMathOperator*{\argmax}{arg\,max}
\DeclareMathOperator*{\argmin}{arg\,min}

\newtheorem{lemma}{Lemma}
\newtheorem{corollary}{Corollary}
\newtheorem{theorem}{Theorem}
\newtheorem{myassum}{Assumption}
\newtheorem{remark}{Remark}
\theoremstyle{definition}

\usepackage{lipsum}

\usepackage{enumitem}
\usepackage{xcolor}
\usepackage[linesnumbered,ruled,vlined]{algorithm2e}

\SetCommentSty{mycommfont}

\SetKwInput{KwInput}{Input}                
\SetKwInput{KwOutput}{Output}

\input{commands.tex}
\begin{document}

%

%




\author{%
  Sanae Amani\\
  University of California, Los Angeles\\
  \texttt{samani@ucla.edu}
  \and
  Christos Thrampoulidis \\
   University of British Columbia, Vancouver\\
  \texttt{cthrampo@ece.ubc.ca}
  }
\title{UCB-based Algorithms for Multinomial Logistic Regression Bandits}
\date{}
\maketitle


\begin{abstract}
Out of the rich family of generalized linear bandits, perhaps the most well studied ones are logisitc bandits that are used in problems with binary rewards: for instance, when the learner/agent tries to maximize the profit over a user that can select one of two possible outcomes (e.g., `click' vs `no-click'). Despite remarkable recent progress and improved algorithms for logistic bandits, existing works do not address practical situations where the number of outcomes that can be selected by the user is larger than two (e.g., `click', `show me later', `never show again', `no click'). In this paper, we study such an extension. We use multinomial logit (MNL) to model the probability of each one of $K+1\geq 2$ possible outcomes (+1 stands for the `not click' outcome): 
we assume that for a learner's action $\x_t$, the user selects one of $K+1\geq 2$ outcomes, say outcome $i$, with a multinomial logit (MNL) probabilistic model with corresponding unknown parameter $\bar\thetab_{\ast i}$. Each outcome $i$ is also associated with a revenue parameter $\rho_i$ and the goal is to maximize the expected revenue. For this problem, we present MNL-UCB, an upper confidence bound (UCB)-based algorithm, that achieves regret $\tilde\Oc(dK\sqrt{T})$ with small dependency on 
problem-dependent constants that can otherwise be arbitrarily large and lead to loose regret bounds. We present numerical simulations that corroborate our theoretical results.


\end{abstract}
\section{Introduction}
Linear stochastic bandits provide simple, yet commonly encountered, models for a variety of sequential decision-making problems under uncertainty. Specifically, linear bandits generalize the classical multi-armed bandit (MAB) problem of $K$ arms that each yields reward sampled independently from an underlying distribution with unknown parameters, to a setting where the expected reward of each arm is a linear function that  depends on the same unknown parameter vector \cite{dani2008stochastic,abbasi2011improved,rusmevichientong2010linearly}. Linear bandits have been successfully applied over the years in online advertising, recommendation services, resource allocation, etc. \cite{lattimore2018bandit}. More recently, researchers have explored the potentials of such algorithms in more complex systems, such as in robotics, wireless networks, the power grid, medical trials, e.g., \cite{li2013medicine,avner2019multi,berkenkamp2016bayesian,sui2018stagewise}. However, linear bandits  fail to model a host of other applications. This has called for  extensions of linear bandits to a broader range of reward structures beyond linear models. One of the leading lines of work addressing these extensions relies on the Generalized Linear Model (GLM) framework of statistic. In GLMs the expected reward associated with an arm $\x$ is given by $\mu(\bar\thetab^T\x)$, where $\bar\thetab\in\mathbb{R}^d$ is
 the system unknown parameter  and $\mu$ is a non-linear link function. Specifically, \emph{logistic bandits}, that are appropriate for modeling binary reward structures, are a special case of generalized linear bandits (GLBs) with $\mu(x)=\frac{1}{1+\exp(-x)}$. UCB-based algorithms for GLBs were first introduced in \cite{filippi2010parametric,li2017provably,faury2020improved}. The same problem, but with a Thompson Sampling- (TS) strategy was also studied in \cite{abeille2017linear,russo2013eluder,russo2014learning,dong2018information}.
Beyond GLMs, an even more general framework for modeling reward is the semi-parametric index model  (see for example \cite{yang2017learning,gamarnik2020estimation} for a list of applications in statistics). A semi-parametric index model relates the reward $y \in\mathbb{R}$ and the action/arm $\x\in\mathbb{R}^d$ as $y=\mu(\bar\thetab_1^T\x,\bar\thetab_2^T\x,\ldots,\bar\thetab_K^T\x)+\epsilon$, where $\mu:\mathbb{R}^K\rightarrow\mathbb{R}$ and  $\bar\thetab_1,\ldots,\bar\thetab_K\in\mathbb{R}^d$ are $K$ system's unknown parameters. GLBs are special cases of this for $K=1$, also known as single-index models (SIM) in statistics. In this paper, we formulate an extension of the problem of binary logistic bandits (i.e., a special case of SIM) to multinomial logit (MNL) bandits, a special case of multi-index models (MIM) to account for settings with more than two possible outcomes on the user choices ($K\geq 1$). For this model, we present an algorithm and a corresponding regret bound. Our algorithmic and analytic contribution is in large inspired by very recent exciting progress on binary logistic bandits by \cite{faury2020improved}.

To motivate  MNL bandits, consider ad placement. When an ad is shown to a user, the user may have several options to react to the ad. For example, she can choose to 1) click on the ad; 2) click on ``show me later''; 3) click on ``never show me this ad''; 4) not click at all, etc. The user selects each of these options based on an unknown probability distribution that inherently involves linear combinations of the selected feature vector denoting the ad and unknown parameters denoting the user’s preferences about the ad. In this setting, each option is associated with a specific notion of reward. The agent’s goal is to determine ads with maximum expected rewards to increase the chance of a successful advertisement.


{\bf Outline.} In Section \ref{sec:formulate}, we formally define the problem. In Sections \ref{sec:mle}, \ref{sec:confidenceset} and \ref{sec:errorbound}, we elaborate on the challenges that the generalization of the Logistic-UCB-1 by \cite{faury2020improved} to the settings with MIM rewards brings to our theoretical analysis. We then summarize our proposed MNL-UCB in Algorithm \ref{alg:MNL-UCB} and provide a regret bound for it in Section \ref{sec:regretbound}. In Section \ref{sec:discuss}, we present a detailed discussion on the challenges and computation of necessary problem-dependent constants. Finally, we complement our theoretical results with numerical simulations in Section \ref{sec:sim}.

\textbf{Notation.} 
We use lower-case letters for scalars, lower-case bold letters for vectors, and upper-case bold letters for matrices. The Euclidean-norm of $\x$ is denoted by $\norm{\x}_2$. For vectors $\x$ and $\y$ $\x^T\y$ denotes their inner product. We denote the Kronecker delta by $\delta_{ij}$. For  matrices $\A$ and $\B$, $\A\otimes\B$ denotes their Kronecker product. For square matrices $\A$ and $\B$, we use $\A\preceq \B$ to denote $\B-\A$ is a positive semi-definite matrix. We denote the minimum and maximum eigenvalues of $\A$ by $\lamin(\A)$ and $\lamax(\A)$. Let $\A\succeq 0$ matrix. The weighted 2-norm of a vector $\boldsymbol \nu$ with respect to $\A$ is defined by $\norm{\boldsymbol \nu}_\A = \sqrt{\boldsymbol \nu^T \A \boldsymbol \nu}$. For positive integers $n$ and $m\leq n$, $[n]$ and $[m:n]$ denote the sets $\{1,2,\ldots,n\}$ and $\{m,\ldots,n\}$, respectively. For any vector $\nub\in\mathbb{R}^{Kd}$, $\bar\nub_i=\nub_{[(i-1)d+1:d]}\in\mathbb{R}^{d}$ denotes the vector containing the $i$-th set of $d$ entries of vector $\nub$. We use $\mathbf{1}$ and $\mathbf{e}_i$ to denote the vector of all $1$'s and the $i$-th standard basis vector, respectively.
Finally, we use standard $\Otilde$ notation for big-Oh notation that ignores logarithmic factors.

\subsection{Problem formulation} \label{sec:formulate}

{\bf Reward Model.} The agent is given a decision set\footnote{Our results extend easily  to time-varying decision sets.} $\Dc \subset \mathbb R^d$. At each round $t$, the agent chooses an action $\x_{t} \in \Dc$ and observes the user purchase decision $y_t\in[K]\cup\{0\}$. Here, $\{0\}$ denotes the ``outside decision'', which means the user did not select any of the presented options. The agent's decision at round $t$ is
based on the information gathered until time $t$, which can be formally encoded in the filtration
$\Fc_t:=\big(\Fc_0,\sigma(\{\x_s,y_s\}_{s=1}^{t-1})\big)$, where $\Fc_0$ represents any prior knowledge. Let each option $i\in[K]$ be associated with an unknown vector $ \bar\thetab_{\ast i} \in \mathbb{R}^d$ and let $\thetab_\ast = [\bar \thetab_{\ast1}^T,\bar\thetab_{\ast2}^T,\ldots,\bar\thetab_{\ast K}^T]^T\in \mathbb{R}^{Kd}$. The user's choice of what to click on is given by a multinomial logit (MNL) choice model. Under this model, the probability distribution of the user purchase decision is given by
\begin{align}
\hspace{-0.05in}\mathbb P(y_t =i | \x_t, \Fc_t):=
     \begin{cases}
      \frac{1}{1+\sum_{j=1}^K\exp(\bar \thetab_{\ast j}^T\x_t)} &\hspace{-0.1in}, \text{if}\ i=0, \\
       \frac{\exp(\bar \thetab_{\ast i}^T\x_t)}{1+\sum_{j=1}^K\exp(\bar \thetab_{\ast j}^T\x_t)} &\hspace{-0.1in}, \text{if}\ i\in[K].
    \end{cases}
\end{align}

When the user clicks on the $i$-th option, a corresponding reward $\rho_i\geq 0$ is revealed to the agent and we set $\rho_0=0$.
Then, the expected reward observed by the agent when she plays action $\x_t$ is
\begin{align*}
    \mathbb E[R_t|\x_t,\Fc_t] = \frac{\sum_{j=1}^K \rho_j \exp(\bar \thetab_{\ast j}^T\x_t)}{1+\sum_{j=1}^K\exp(\bar \thetab_{\ast j}^T\x_t)} = \boldsymbol{\rho}^T\z(\x_t,\boldsymbol \theta_\ast),
\end{align*}
where $\boldsymbol{\rho} = [\rho_1,\rho_2,\ldots,\rho_K]^T$, $\z(\x_t,\boldsymbol \theta_\ast)=[z_1(\x_t,\boldsymbol \theta_\ast),z_2(\x_t,\boldsymbol \theta_\ast),\ldots,z_K(\x_t,\boldsymbol \theta_\ast)]^T$, and
\begin{align}\label{eq:z_def}
 z_i(\x_t,\thetab_\ast)= \mathbb P(y_t =i | \x_t, \Fc_t),~\forall i\in[K]\cup\{0\}.   
\end{align}
Note that $\mathbb E[R_t|\x_t,\Fc_t]=\mu(\bar\thetab_{\ast 1}^T\x,\bar\thetab_{\ast2}^T\x,\ldots,\bar\thetab_{\ast K}^T\x)$ is \emph{not} directly a generalized linear model, i.e., a function $\mu(\bar\thetab_\ast^T\x_t)$, but rather it is a \emph{multi-index model}, where $\mu:\mathbb{R}^K\rightarrow \mathbb{R}$.

{\bf Goal.}
Let $T$ be the total number of rounds and $\x_\ast$ be the optimal action that maximizes the reward in expectation, i.e., $\x_\ast \in \argmax_{\x\in\Dc} \boldsymbol{\rho}^T\z(\x,\boldsymbol \theta_\ast)$. The agent's goal is to minimize the \textit{cumulative pseudo-regret} defined by

\begin{equation}
    R_T = \sum_{t=1}^T \boldsymbol{\rho}^T\z(\x_\ast,\boldsymbol \theta_\ast)-\boldsymbol{\rho}^T\z(\x_t,\boldsymbol \theta_\ast).
\end{equation}

\subsection{Contributions}
We study MNL logistic regression bandits, a generalization of binary logistic bandits, that address applications where the number of outcomes that can be selected by the user is larger than two. The probability of any possible $K+1>2$ outcomes ($+1$ stands for the `not click' outcome aka ``outside decision'') is modeled using a multinomial logit (MNL) model. For this problem:

\noindent$\bullet$~We identify a critical parameter $\kappa$, which we interpret as the degree of (non)-smoothness (less smooth for larger values of $\kappa$) of the MNL model over the agent's decision set. We prove that $\kappa$ scales exponentially with the size of the agent's decision set creating a challenge in the design of low-regret algorithms, similar to the special binary case previously studied in the literature.


\noindent$\bullet$~We develop a UCB-type algorithm for MNL logistic regression bandits. At every step, the algorithm decides on the inclusion of a K-tuple of parameter vectors $\bar\thetab_{\ast 1},\ldots,\bar\thetab_{\ast K}$ in the confidence region in a way that captures the local smoothness of the MNL model around this K-tuple and past actions. We show that this is critical for the algorithm's favorable regret performance in terms of $\kappa$.

\noindent$\bullet$ Specifically, we prove that the regret of our MNL-UCB scales as $\tilde\Oc(dK\sqrt{\kappa}\sqrt{T}).$ Instead, we show that a confidence ellipsoid that fails to capture local dependencies described above results in regret that scales linearly with $\kappa$ rather than with $\sqrt{\kappa}.$ 
Moreover, our regret bound scales optimally in terms of the number of actions $K$.

\noindent$\bullet$~We complement our theoretical results with numerical simulations and corresponding discussions on the performance of our algorithm.





\subsection{Related works} \label{sec:relatedwork}
{\bf Generalized Linear Bandits.}
GLBs were studied in \cite{filippi2010parametric,li2017provably,abeille2017linear,russo2013eluder,russo2014learning,dong2018information} where the stochastic reward is modeled through an appropriate strictly increasing link function $\mu$. All these works provide regret bounds $\Otilde(\kappa\sqrt{T})$, where the multiplicative factor $\kappa$ is a problem-dependent constant and characterizes the degree of non-linearity of the link function.

{\bf Logistic Bandits.}
In the recent work \cite{faury2020improved}, the authors focused on the \emph{logistic bandit} problem as a special case of GLBs. By introducing a novel Bernstein-like
self-normalized martingale tail-inequality, they reduced the dependency of the existing GLB algorithms' regret bounds on the constant $\kappa$ by a factor of $\sqrt{\kappa}$ and obtained a $\Oc(d\sqrt{\kappa T})$ regret for the logistic bandit problem. They further discussed the crucial role of $\kappa$, which can be arbitrarily large as it scales exponentially with the size of the decision set, on the performance of existing algorithms. Motivated by such considerations, with careful algorithmic designs, they achieved to drop entirely the dependence on $\kappa$ leading to a regret of $\Otilde(d\sqrt{T})$.

{\bf Multinomial Logit Bandits.} In a different line of work, \cite{agrawal2017thompson,agrawal2019mnl,wang2018near,oh2019thompson,chen2018dynamic,dong2020multinomial} used the multinomial logit choice model to address the \emph{dynamic assortment selection} problem, which is a combinatorial variant of the bandit problem. In this problem, the agent chooses a so-called assortment which is a subset of a set $\Sc=[N]$ of $N$ items. At round $t$, feature vectors $\x_{it}$, for every item $i\in\Sc$, are revealed to the agent, and given this contextual information, the agent selects an assortment $S_t\subset\Sc$ and observes the user choice $y_t=i$, $i\in S_t\cup \{0\}$ where $\{0\}$ corresponds to the user not selecting any item in $S_t$. The user choice is given by a MNL model with an unknown parameter $\bar\thetab_\ast\in\mathbb{R}^d$ such that the probability that the user selects item $i\in S_t$ is $\frac{\exp(\bar\thetab_\ast^T\x_{it})}{1+\sum_{i\in S_t}\exp(\bar\thetab_\ast^T\x_{it})}$. Furthermore, a revenue parameter denoted by $r_{it}$ for each item $i$ is also revealed at round $t$. The goal of the agent is to offer assortments with size at most $K$ to maximize the expected cumulative revenue or to minimize the cumulative regret $\sum_{t=1}^T R_t(S_t^\ast,\bar\thetab_\ast)-R_t(S_t,\bar\thetab_\ast)$, 
where
$
    R_t(S_t,\bar\thetab_\ast):={\sum_{i\subset S_t}r_{it}\exp(\bar\thetab_{\ast}^T\x_{it})}/\big({1+\sum_{i\in S_t}\exp(\bar\thetab_{\ast}^T\x_{it})}\big).
$
Finally, the closely related paper \cite{cheung2017thompson} studies a problem where at each round, the agent observes a user-specific context based on which, it recommends a set of items to the user. The probability distribution of each one of the items in that set being selected by the user is given by an MNL model. This problem can be categorized as an online assortment optimization problem. Despite similarities in the use of an MNL model, there are certain differences between \cite{cheung2017thompson} and our paper in terms of problem formulation. In our setting, the user may have multiple reactions (one of $K+1$ options), to a single selected item. In contrast, in \cite{cheung2017thompson}, the agent must select a set of items to each of which the user reacts by either clicking or not clicking. Also, here the probability distribution of different user actions remains the same at all rounds, while in \cite{cheung2017thompson} the response to an item from the recommended set depends on the other items in that set. We defer to future work studying implications of our techniques to the interesting setting of \cite{cheung2017thompson}.

\section{Multnomial Logit UCB Algorithms}\label{sec:mnlucb}


In this section, we introduce two key quantities: (i) $\thetab_t$, an estimate of $\thetab_\ast$; (ii)  $\epsilon_t(\x)$, an exploration bonus for each $\x\in\Dc$ at each round $t\in[T]$. Based on these, we design a UCB-type algorithm, called MNL-UCB. At each round $t$, the algorithm computes an estimate $\boldsymbol \theta_t$ of $\boldsymbol\theta_\ast$, that we  present in Section \ref{sec:errorbound}. For each $\x\in\Dc$ and $t\in[T]$, let $\epsilon_t(\x)$ be such that the following holds with high probability:
\begin{align}\label{eq:Delta(x,thetabt)}
    \Delta(\x,\thetab_t):=|\boldsymbol{\rho}^T\z(\x,\thetab_\ast)-\boldsymbol{\rho}^T\z(\x,\thetab_t)| \leq \epsilon_t(\x).
\end{align}
At round $t$, having knowledge of $\epsilon_t(\x)$, the agent computes the following upper bound on the expected reward $\boldsymbol{\rho}^T\z(\x,\thetab_\ast)$ for all $\x\in\Dc$:
\begin{align}\label{eq:upperbound}
    \boldsymbol{\rho}^T\z(\x,\thetab_\ast)\leq \boldsymbol{\rho}^T\z(\x,\thetab_t) + \epsilon_t(\x).
\end{align}

Then, the learner follows a UCB decision rule to select an action $\x_t$ according to the following rule:
\begin{align}\label{eq:ucbdecisionrule}
    \x_t:=\argmax_{\x\in\Dc} \boldsymbol{\rho}^T\z(\x,\thetab_t) + \epsilon_t(\x).
\end{align}
To see how the UCB decision rule in \eqref{eq:ucbdecisionrule} helps us control the cumulative regret, we show how it controls the instantaneous regret by the following standard argument \cite{abbasi2011improved,faury2020improved}:

\begin{align}
    r_t &= \boldsymbol{\rho}^T\z(\x_\ast,\thetab_\ast)-\boldsymbol{\rho}^T\z(\x_t,\thetab_\ast)\nn\\
    &\leq\boldsymbol{\rho}^T\z(\x_\ast,\thetab_t)+\epsilon_t(\x_\ast)-\boldsymbol{\rho}^T\z(\x_t,\thetab_\ast)\tag{Eqn.~\eqref{eq:upperbound}}\\
    &\leq \boldsymbol{\rho}^T\z(\x_t,\thetab_t)-\boldsymbol{\rho}^T\z(\x_t,\thetab_\ast)+\epsilon_t(\x_t)\tag{Eqn.~\eqref{eq:ucbdecisionrule}}\\
    &\leq2\epsilon_t(\x_t)\tag{Eqn.~\eqref{eq:Delta(x,thetabt)}}.
\end{align}
In view of this, our goal is to design an algorithm that appropriately chooses the estimator $\thetab_t$ and the exploration bonus $\epsilon_t(\x)$ such that its regret is sub-linear.

\subsection{Maximum likelihood estimate}\label{sec:mle}

The problem of estimating $\thetab_\ast$ at round $t$ given $\Fc_t$ is identical to a multi-class linear classification problem, where $\bar\thetab_{\ast i}$ is the ``classifier'' for class $i$.
A natural way to compute the estimator of the unknown parameter $\thetab_\ast$ of the MNL model given $\Fc_t$ is to use the maximum likelihood principle. At round $t$, the regularized log-likelihood (aka negative cross-entropy loss) with regularization parameter $\la>0$ writes
\begin{align}\label{eq:loglikelihood}
    \hspace{-0.2in}\Lc_t^\la(\thetab):=\sum_{s=1}^{t-1}\sum_{i=0}^K \mathbbm{1} \{y_s=i\}\log\left(z_i(\x_s,\thetab)\right)-\frac{\la}{2}\norm{\thetab}_2^2.
\end{align}
Then, the maximum likelihood estimate of $\thetab_\ast$ is defined
as $\hat\thetab_t:=\argmax_{\thetab}\Lc_t^\la(\thetab)$. Taking the gradient of \eqref{eq:loglikelihood} with respect to $\thetab$ we obtain

\begin{align}\label{eq:gradoflogloss}
    \nabla_\thetab\Lc_t^\la(\thetab):= \sum_{s=1}^{t-1} [\m_{s}-\z(\x_s,\thetab)]\otimes \x_s-\la \thetab,
\end{align}
where $\m_s$ is the `one-hot encoding' vector of the user's selection at round $s$, i.e., 
$
   \m_s:=\left[\mathbbm{1} \{y_s=1\},\ldots,\mathbbm{1} \{y_s=K\}\right]^T.
$
It will also be convenient to define the Hessian of $-\Lc_t^\la(\thetab)$:
\begin{align}\label{eq:Hessian}
    \Hb_t(\thetab):= \la I_{Kd}+\sum_{s=1}^{t-1}\A(\x_s,\thetab)\otimes \x_s\x_s^T,
\end{align}
where 
$\A(\x,\thetab)_{ij}:=z_i(\x,\thetab)\left(\delta_{ij}-z_j(\x,\thetab)\right)$ for all $i,j\in[K]$. Equivalently, in matrix form
\begin{align}\label{eq:A}
   \A(\x,\thetab):=\diag(\z(\x,\thetab))-\z(\x,\thetab)\z^T(\x,\thetab).
\end{align}
Note here that $\A(\x,\thetab)$ is a matrix function that depends on $\thetab$ and $\x$ via the inner products $\bar\thetab_{1}^T\x,\bar\thetab_{2}^T\x,\ldots,\bar\thetab_{K}^T\x$. Also, the matrix $\A(\x,\thetab)$ has nice algebraic properties (discussed more in Section \ref{sec:discuss}) that turn our to be critical in the execution and analysis of our algorithm. 

Now, we introduce the necessary assumptions on the problem structure, under which our theoretical results hold.

\begin{myassum}[Boundedness]\label{assum:boundedness} Without loss of generality, $\norm{\x}_2\leq 1$ for all $\x\in \Dc$. Furthermore, $\boldsymbol \theta_\ast\in \Theta :=\{\thetab\in \mathbb{R}^{Kd}: \norm{\thetab}_2\leq S\}$ and $\norm{\boldsymbol{\rho}}_2\leq R$. Both upper bounds $S$ and $R$ are known to the agent.
\end{myassum}
The assumption that $S$ is known is standard in the literature of GLBs. Knowledge of $R$ is also reasonable to assume because $\rho_i$'s represent the revenue parameters that are typically known or set by the system operator.
\begin{myassum}[Problem-dependent constants]\label{assum:problemdependentconstants}
There exist strictly positive constants $0<L<\infty$ and $0<\kappa<\infty$ such that $\sup_{\x\in\Dc,\thetab\in\Theta}\lamax\left(\A(\x,\thetab)\right):= L$ and $\inf_{\x\in\Dc,\thetab\in\Theta}\lamin\left(\A(\x,\thetab)\right):= \frac{1}{\kappa}$.
\end{myassum}
We comment further on the knowledge of $\kappa$ and $L$ in Section \ref{sec:discuss}. Here, we note that the constant $\kappa$ is reminiscent of the corresponding quantity in binary logistic bandits which is  defined  accordingly as $\kappa:=   \sup_{\x\in\Dc,\norm{\bar\thetab}_2\leq S} 1/\dot{\mu}(\bar\thetab^T\x)$, where $\dot{\mu}$ is the first derivative of the logistic function $\mu(x)=1/(1+\exp(-x))$.  As \cite{filippi2010parametric,li2017provably,faury2020improved} have shown, this quantity plays a key role in characterizing the behavior of binary ($K=1$) logit bandit algorithms. In this paper, we will show that the proper analogue of this quantity to multinomial ($K>1$) logit bandit algorithms is the parameter $\kappa$  defined in Assumption \ref{assum:problemdependentconstants}.


\subsection{Confidence set around \texorpdfstring{$\thetab_\ast$}{k}}\label{sec:confidenceset}

In this section, we introduce a confidence set $\Cc_t(\delta)$ that will include $\thetab_\ast$ with high probability thus allowing us to upper bound $\Delta(\x,\thetab_t)$ in the following subsection. We start by defining the key quantity 
\begin{align}\label{eq:g_t}
    \g_t(\thetab):=\la \thetab+\sum_{s=1}^{t-1}\z(\x_s,\thetab)\otimes\x_s.
\end{align}
To see why this is useful, note that by the first-order optimality condition $\nabla_\thetab\Lc_t^\la(\hat\thetab_t)=\mathbf{0}$ we have  $\thetab_\ast$ satisfying the following:
\begin{align*}
    \g_t(\thetab_\ast)-\g_t(\hat {\boldsymbol\theta}_{t})=\la\thetab_\ast+\s_t, 
\end{align*}
with $
      \s_t:=\sum_{s=1}^{t-1}\left(\z(\x_s,\thetab_\ast)-\m_s\right)\otimes\x_s.$
This in turn motivates us to define a confidence set $\Cc_{t}$ at the beginning of each round $t\in[T]$ such that
\begin{align}\label{eq:thetaconfidenceset}
    \hspace{-0.2in}\Cc_{t}(\delta):=\{\thetab \in \Theta: \norm{{\g_t(\thetab)-\g_t(\hat {\boldsymbol\theta}_{t}})}_{\Hb^{-1}_{t}(\thetab)} \leq \beta_t(\delta)\}, 
\end{align}
where $\beta_t(\delta)$ is chosen as in Theorem \ref{thm:confidenceset} below to guarantee that $\boldsymbol\theta_\ast \in \Cc_{t}(\delta)$ with high probability $1-\delta$.

\begin{theorem}[Confidence set]\label{thm:confidenceset}
Let the Assumption \ref{assum:boundedness} hold and for $\delta\in(0,1)$, define
 \begin{equation}
    \beta_t(\delta):= \frac{K^{3/2}d}{\sqrt{\la}}\log\big(1+\frac{t}{d\la}\big)+\frac{\sqrt{\la/K}}{2}+\frac{2K^{3/2}d}{\sqrt{\la}}\log(\frac{2}{\delta})+\sqrt{\la}S\nn.
 \end{equation}
Then with probability at least $1-\delta$, for all $t\in[T]$ it holds that $\boldsymbol\theta_\ast\in \Cc_{t}(\delta)$.
\end{theorem}

Once we have properly identified the key quantities $\g_t(\thetab_\ast)$ and $\Hb_t(\thetab)$, the proof of the theorem above rather naturally extends Lemma 1 in \cite{faury2020improved}. Compared to the special case $K=1$  studied in \cite{faury2020improved}, extra care is needed here to properly track the scaling of $\beta_t(\delta)$ with respect to the new parameter $K$ that is  of interest for us. The details are deferred to Appendix \ref{sec:confidencesetproof}.
To a great extent, the similarities between the binary and multinomial cases end here. It will turn out that bounding $\Delta(\x,\thetab_t)$, for an appropriate choice of $\thetab_t$, is significantly more intricate here than in the binary case. Our main technical contribution towards this direction is given in the lemmas presented in Section \ref{sec:proofsketch} that are used to prove the following key result.


\begin{lemma}\label{lemm:pred}
Let Assumptions \ref{assum:boundedness} and \ref{assum:problemdependentconstants} hold. For all $\x\in\Dc$, $t\in[T]$ and $\thetab\in \Cc_t(\delta)$, with probability at least $1-\delta$: 
\begin{align} \label{eq:upperboundonDelta}
    \Delta(\x,\thetab)\leq 2RL\beta_t(\delta)\sqrt{\kappa \left(1+2S\right)}\norm{\x}_{\Vb_t^{-1}}.
\end{align}
\end{lemma}
The complete proof is in Appendix \ref{sec:predlemmaproof}. In the following section, we give a proof sketch.

\subsection{Proof sketch of Lemma \ref{lemm:pred}}\label{sec:proofsketch}
To prove Lemma \ref{lemm:pred}, we will use the high probability confidence set $\Cc_t(\delta)$ in \eqref{eq:thetaconfidenceset} paired with the problem setting's properties encapsulated in Assumptions \ref{assum:boundedness} and \ref{assum:problemdependentconstants}. 
To see the key challenges in establishing \eqref{eq:upperboundonDelta}, consider the following. By definition of $\Delta(\x,\thetab)$ in \eqref{eq:Delta(x,thetabt)}, Cauchy-Schwarz inequality, and Assumption \ref{assum:boundedness}, for any $\x\in\Dc$, $t\in[T]$, and $\thetab\in\Cc_t(\delta)$:
\begin{align}\label{eq:firstimportant}
    \Delta(\x,\thetab)\leq R\|\z(\x,\thetab_\ast)-\z(\x,\thetab)\|_2.
\end{align}
Thus, our goal becomes relating $\|\z(\x,\thetab_\ast)-\z(\x,\thetab)\|$ to $\norm{\g_t(\thetab_\ast)-\g_t(\hat\thetab_t)}_{\Hb_t^{-1}(\thetab_\ast)}$ and/or $\norm{\g_t(\hat\thetab_t)-\g_t(\thetab)}_{\Hb_t^{-1}(\thetab)}$. The reason is that the two latter quantities are both known to be bounded by $\beta_t(\delta)$ with high probability since $\thetab_\ast\in\Cc_t(\delta)$ with probability at least $1-\delta$ (cf. Theorem \ref{thm:confidenceset}). We accomplish our goal in three steps. First, in Lemma \ref{lemm:z(1)-z(2)original}, we connect $\z(\x,\thetab_1)-\z(\x,\thetab_2)$ to $\thetab_1-\thetab_2$ for any $\x\in\mathbb{R}^d$ and $\thetab_1,\thetab_2\in\mathbb{R}^{Kd}$.

\begin{lemma}\label{lemm:z(1)-z(2)original}
For any $\x\in\mathbb{R}^d$ and $\thetab_1,\thetab_2\in\mathbb{R}^{Kd}$, recall the definition of matrix $\A(\x,\thetab)$ in \eqref{eq:A} and define
\begin{align}\label{eq:B(x,1,2)original}
    \B(\x,\thetab_1,\thetab_2):= \int_{0}^1\A(\x,v\thetab_1+(1-v)\thetab_2)dv.
\end{align}
Then, we have
\begin{align*}
    \z(\x,\thetab_1)-\z(\x,\thetab_2)=\left[\B(\x,\thetab_1,\thetab_2)\otimes \x^T\right] (\thetab_1-\thetab_2).
\end{align*}
\end{lemma}
The proof of the lemma above in Appendix \ref{sec:necessarypredlemmaproof} relies on a proper application of the mean-value Theorem. Next, in Lemma \ref{lemm:g_t(1)-g_t(2)original} below, we relate $\thetab_1-\thetab_2$ to $\g_t(\thetab_1)- \g_t(\thetab_2)$.
\begin{lemma}\label{lemm:g_t(1)-g_t(2)original}
Let
\begin{align}\label{eq:G_t_original}
    \Gb_t(\thetab_1,\thetab_2):=\la I_{Kd}+\sum_{s=1}^{t-1}\B(\x_s,\thetab_1,\thetab_2)\otimes \x_s\x_s^T.
\end{align}
Then, for any $\thetab_1,\thetab_2\in\mathbb{R}^{Kd}$, we have
\begin{align}\label{eq:g_t-g_t_original}
    \g_t(\thetab_1)- \g_t(\thetab_2) = \Gb_t(\thetab_1,\thetab_2)(\thetab_1-\thetab_2).
\end{align}
\end{lemma}
The proof in Appendix \ref{sec:necessarypredlemmaproof} uses the definition of $\g_t(\thetab)$, Lemma \ref{lemm:z(1)-z(2)original} and a proper application of the mixed-product property of the Kronecker product. Also, note that the matrix $\Gb_t(\thetab_1,\thetab_2)$ is positive definite and thus, it is invertible.

Now, combining Lemmas \ref{lemm:z(1)-z(2)original} and \ref{lemm:g_t(1)-g_t(2)original}, our new goal becomes bounding $\norm{\left[\B(\x,\thetab_\ast,\thetab)\otimes \x^T\right]\Gb_t^{-1}(\thetab_\ast,\thetab)\left(\g_t(\thetab_\ast)-\g_t(\thetab)\right)}_2$. Our key technical contribution towards achieving this goal is establishing good bounds for the spectral norm
 $\norm{\left[\B(\x,\thetab_\ast,\thetab)\otimes \x^T\right]\Gb_t^{-1/2}(\thetab_\ast,\thetab)}_2$ and the weighted Euclidean norm $\norm{\g_t(\thetab_\ast)-\g_t(\thetab)}_{\Gb_t^{-1}(\thetab_\ast,\thetab)}$.

We start by briefly explaining how we bound the spectral norm above; see Appendix \ref{sec:completingproofoflemmapred} for details. By using the cyclic property of the maximum eigenvalue and the mixed-product property of the Kronecker product, it suffices to bound $\la_{\max}\left(\Gb_t^{-1/2}\left(\left(\B^T\B\right)\otimes \left(\x\x^T\right)\right)\Gb_t^{-1/2}\right)$, where we denote $\B=\B(\x,\thetab_\ast,\thetab)$ and $\Gb_t=\Gb_t(\thetab_\ast,\thetab)$ for simplicity. There are two essential ideas to do so. First, thanks to our Assumption \ref{assum:problemdependentconstants}, which upper bounds the eigenvalues of $\mathbf{A}(\mathbf{x},\thetab)$ by $L$, and by recalling the definition of $\mathbf{B}(\mathbf{x},\thetab_1,\thetab_2)$, we manage to show that $(\B\B^T)\otimes \left(\x\x^T\right)\preceq L^2 \left(I_K\otimes\x\right)\left(I_K\otimes\x^T\right)$. Our second idea is to relate the matrix $\Gb_t$ to the Gram matrix of actions
\begin{align}\label{eq:Vt}
    \Vb_t:=\kappa\la I_{d}+\sum_{s=1}^{t-1} \x_s\x_s^T.
\end{align}
Specifically, using our Assumption \ref{assum:problemdependentconstants} that  the minimum eigenvalue of $\mathbf{A}(\mathbf{x},\thetab)$ is lower bounded by $1/\kappa$ and standard spectral properties of the Kronecker product, we prove in Lemma \ref{lemm:positivedefiniteVt} in the appendix that $\Gb_t\succeq \frac{1}{\kappa} I_K\otimes \Vb_t$ or $\Gb_t^{-1}\preceq {\kappa} I_K\otimes \Vb_t^{-1}.$ By properly combining the above, we 
achieve the following convenient upper bound: 
\begin{align}
    \norm{\big[\B(\x,\thetab_\ast,\thetab)\otimes \x^T\big]\,\Gb_t^{-1/2}(\thetab_\ast,\thetab)}_2\leq L\sqrt{\kappa}\norm{\x}_{\Vb_t^{-1}}.
\end{align}

To see why this is useful, note that compared to $\Hb_t(\thetab)$ in \eqref{eq:Hessian}, which is a `matrix-weighted' version of the Gram matrix, the definition of $\Vb_t$ in \eqref{eq:Vt} is the same as the definition of the Gram matrix in linear bandits. Thus, we are now able to use standard machinery in \cite{abbasi2011improved} to bound  $\sum_{t=1}^T\norm{\x_t}_{\Vb_t^{-1}}$. 

Finally, we discuss how to control the remaining weighted Euclidean norm $\norm{\g_t(\thetab_\ast)-\g_t(\thetab)}_{\Gb_t(\thetab_\ast,\thetab)^{-1}}$. By adding and subtracting $\g_t(\hat\thetab_t)$ to and from the argument inside the norm, 
it suffices to bound
\begin{align}\label{eq:keyquant2}
    \norm{\g_t(\thetab_\ast)-\g_t(\hat\thetab_t)}_{\Gb_t(\thetab_\ast,\thetab)^{-1}}+\norm{\g_t(\hat\thetab_t)-\g_t(\thetab)}_{\Gb_t(\thetab_\ast,\thetab)^{-1}}.
\end{align}
We are able to do this by exploiting the definition of the confidence set $\Cc_t(\delta)$ in \eqref{eq:thetaconfidenceset} by first relating the $\Gb_t^{-1}(\thetab_\ast,\thetab)$--norms with those in terms of  $\Hb_t^{-1}(\thetab_\ast)$ and $\Hb_t^{-1}(\thetab)$. To do this, we rely on the \emph{generalized self concordance} property of the strictly convex \emph{log-sum-exp} (${\rm lse}: \mathbb{R}^{K}\rightarrow\mathbb{R}\cup{\infty}$)  function \cite{tran2015composite}
\begin{align}\label{eq:lse}
    {\rm lse}(\s):=\log(1+\sum_{i=1}^{K}\exp(\s_i)).
\end{align}
Notice that our $\z(\x,\thetab)$ is the gradient of the
log-sum-exp function at point $[\bar\thetab_1^T\x,\bar\thetab_2^T\x,\ldots,\bar\thetab_K^T\x]^T$, that is
\begin{align}
    \z(\x,\thetab)=\nabla{\rm lse}\left([\bar\thetab_1^T\x,\bar\thetab_2^T\x,\ldots,\bar\thetab_K^T\x]^T\right).
\end{align}
Thanks to the generalized self-concordance property of the lse \eqref{eq:lse}, upper bounds on its Hessian matrix (essentially the matrix $\Hb_t(\thetab)$) have been developed in \cite{tran2015composite,sun2019generalized}. Proper use of such bounds leads to lower bounds on $\Gb_t(\thetab_\ast,\thetab)$ as follows.


\begin{lemma}[Generalized self-concordance]\label{lemm:generalselforiginal}
For any $\thetab_1,\thetab_2\in\Theta$, we have
$
    (1+2S)^{-1} \Hb_t(\thetab_1)\preceq  \Gb_t(\thetab_1,\thetab_2)$ and $(1+2S)^{-1} \Hb_t(\thetab_2)\preceq  \Gb_t(\thetab_1,\thetab_2).$

\end{lemma}
The proof is given in Appendix \ref{sec:necessarypredlemmaproof}.
Finally, plugging the above lower bounds on matrix $\Gb_t$ into \eqref{eq:keyquant2} for $\thetab_\ast$ and $\thetab$ gives the final bound on $\Delta(\x,\thetab)$ in Lemma \ref{lemm:pred}.

\subsection{Error bound on \texorpdfstring{$\Delta(\x,\thetab_t)$}{k}}\label{sec:errorbound}
In this section, we specify $\thetab_t$ and $\epsilon_t(\x)$ for all $t\in[T]$ and $\x \in\Dc$. In view of Lemma \ref{lemm:pred}, the exploration bonus $\epsilon_t(\x)$ can be set equal to the RHS of \eqref{eq:upperboundonDelta}, only if $\thetab_t\in\Cc_t(\delta)$. Recall from the definition of $\Cc_t(\delta)$ in \eqref{eq:thetaconfidenceset} that for $\thetab_t\in\Cc_t(\delta)$, it must be that $\thetab_t\in\Theta$. Since the ML estimator $\hat\thetab_t$ does not necessarily satisfy $\hat\thetab_t\in\Theta$, we introduce the following ``feasible estimator":
\begin{align}\label{eq:feasibleestimator}
    \thetab_t:=\argmin_{\thetab\in\Theta}\norm{\g_t(\thetab)-\g_t(\hat\thetab_t)}_{\Hb_t^{-1}(\thetab)},
\end{align}
which is guaranteed to be in the confidence set $\Cc_t(\delta)$ for all $t\in[T]$ since 
 $\|{\g_t(\thetab_t)-\g_t(\hat\thetab_t)}\|_{\Hb_t^{-1}(\thetab)}\leq \|{\g_t(\thetab_\ast)-\g_t(\hat\thetab_t)}\|_{\Hb_t^{-1}(\thetab)}\leq \beta_t(\delta)$ and $\thetab_t\in\Theta$. Thus, we have proved the following.
\begin{corollary}[Exploration bonus]\label{corr:theonly}
For all $\x\in\Dc$ and $t\in[T]$, with probability at least $1-\delta$, we have
\begin{align}\label{eq:Deltaaa}
    \hspace{-0.2in}\Delta(\x,\thetab_t)\leq \epsilon_t(\x):= 2RL\beta_t(\delta)\sqrt{\kappa \left(1+2S\right)}\norm{\x}_{\Vb_t^{-1}}.
\end{align}
\end{corollary}

With these, we are now ready to summarize MNL-UCB in Algorithm \ref{alg:MNL-UCB}.

\begin{algorithm}[t]
\DontPrintSemicolon
\For{$t=1,\ldots,T$}
{
Compute $\thetab_t$ as in \eqref{eq:feasibleestimator}.\;
Compute $\x_t:=\argmax_{\x\in\Dc} \boldsymbol{\rho}^T\z(\x,\thetab_t) + \epsilon_t(\x)$ with $\epsilon_t(\x)$ defined in \eqref{eq:Deltaaa}.\label{line:ucb} \;
Play $\x_t$ and observe $y_t$.
}
\caption{MNL-UCB}
\label{alg:MNL-UCB}
\end{algorithm}

\begin{remark}
The projection step in \eqref{eq:feasibleestimator} is similar to the ones used in \cite{filippi2010parametric} and \cite{faury2020improved} for binary logistic bandits. In particular, this step in \cite{filippi2010parametric} involves norms with respect to $\Vb_t$ instead of $\Hb_t(\thetab)$ (for $K=1$). All of these involve complicated non-convex optimization problems. Empirically, we observe that it occurs frequently that $\hat\thetab_t\in\Theta$. Then,  $\thetab_t=\hat\thetab_t$ and these complicated projection steps do not need to be implemented.
\end{remark}

\subsection{Regret bound of MNL-UCB}\label{sec:regretbound}

In the following theorem, we state the regret bound of MNL-UCB as our main result.
\begin{theorem}[Regret of MNL-UCB]\label{thm:regretbound}
Fix $\delta\in(0,1)$. Let Assumptions \ref{assum:boundedness} and \ref{assum:problemdependentconstants} hold. Then, with probability at least $1-\delta$, it holds that
\begin{align}
   R_T \leq 4RL\beta_T(\delta)
   \sqrt{2\max(1,\frac{1}{\la\kappa})\kappa \left(1+2S\right)dT\log(1+\frac{T}{\kappa \la d})}\nn.
\end{align}
In particular, choosing $\la=Kd\log(T)$ yields
    $R_T=\Oc\left(RLKd\log(T)\sqrt{\kappa T}\right).\nn$
\end{theorem}

Now we comment on the regret order with respect to the key problem parameters $T,d,K$ and $\sqrt{\kappa}$. With respect to these the theorem shows a scaling of the regret of Algorithm \ref{alg:MNL-UCB} as $\Oc(Kd\log(T)\sqrt{\kappa T})$. 
Specifically, for $K=1$ we retrieve the exact same scaling as in Theorem 2 of \cite{faury2020improved} for the binary case in terms of $T,d$ and $\kappa$. In particular, the bound is optimal with respect to the action-space dimension $d$ as $\Oc(d)$ is the optimal order in the simpler setting of linear bandits \cite{dani2008stochastic}. Of course, compared to \cite{faury2020improved} our result applies for general $K\geq1$ and implies a linear scaling with the number $K$ of possible outcomes that can be selected by the user. In fact, our bound suggests that the performance of our algorithm on a $K$-multinomial problem is of the same order as the performance of the algorithm of \cite{faury2020improved} for a binary problem with dimension $Kd$ instead of $d$. On the one hand, this is intuitive since the MNL reward model indeed involves $Kd$ unknown parameters. Thus, we cannot expect regret better than $\Oc(Kd)$ for $Kd$ unknown parameters. On the other hand, the MNL is a special case of multi-index models rather than a GLM. Thus, it is a-priori unclear whether it  differs from a binary logistic model with $Kd$ parameters in terms of regret performance. In fact, our proof does \emph{not} treat the MNL reward model as a GLM. Despite that, we are able to prove the optimal linear order $\Oc(K)$. Finally, as previously mentioned, the parameter $\kappa$ defined in Assumption \ref{assum:problemdependentconstants} generalizes the corresponding parameter for binary bandits. As in Theorem 2 of \cite{faury2020improved} our bound for the multinomial case scales with $\sqrt{\kappa}.$ Next, we comment on how this scaling is non-trivial and improves upon more standard approaches.

\begin{remark}
Using the same tools as for GLM-UCB in \cite{filippi2010parametric} for single-index reward models ($K=1$), we can also define the following alternative confidence set around the parameter $\thetab_\ast$ of the MNL reward model:
\begin{align}\label{eq:badconfidenceset}
    \Ec_{t}(\delta):=\Big\{\thetab \in \Theta: \|{\thetab-\tilde {\boldsymbol\theta}_{t}}\|_{\tilde\Vb^{-1}_{t}} \leq \kappa\gamma_t(\delta)\Big\},
\end{align}
where $\gamma_t(\delta)$ is a slowly increasing function of $t$ with similar order as $\beta_t(\delta)$ (see Lemma \ref{lemm:badconfidenceset} in appendix), $\tilde\Vb_t = I_K\otimes\Vb_t$, and
\begin{align}
    \tilde\thetab_t:=\argmin_{\thetab\in\Theta}\,\|{\g_t(\thetab)-\g_t(\hat\thetab_t)}\|_{\tilde\Vb_t^{-1}}.
\end{align}
Due to the appearance of an extra $\kappa$ factor above compared to \eqref{eq:thetaconfidenceset}, relying on $\Ec_t(\delta)$ \eqref{eq:badconfidenceset} in our analysis would lead to the following error bound (see Appendix \ref{sec:looser_pf_2}):
\begin{align}\label{eq:badexplorationbonus}
    \Delta(\x,\tilde\thetab_t)\leq \tilde\epsilon_t(\x):=2RL\kappa\gamma_t(\delta)\norm{\x}_{\Vb_t^{-1}}.
\end{align}
This bound is significantly looser compared to our bound in \eqref{eq:Deltaaa} since the parameter $\kappa$ can become arbitrarily large depending on the size of set $\Dc\times \Theta$.
\end{remark}


\section{Discussion on \texorpdfstring{$\kappa$}{k} and \texorpdfstring{$L$}{k}}\label{sec:discuss}
Knowledge of the problem-dependent constants $\kappa$ and $L$ is required to implement MNL-UCB as they appear in the definition of $\epsilon_t(\x)$ in \eqref{eq:Deltaaa}. While specifying their true values (defined in Assumption \ref{assum:problemdependentconstants}) requires solving non-convex optimization problems in general, here we present computationally efficient ways to obtain simple upper bounds. Furthermore, towards confirming our claim that $\kappa$ can scale poorly with the size of the set $\Dc\times\Theta$ we derive an appropriate lower bound for it. 
The upper bound on $L$ is rather straightforward and it can be easily checked that $L\leq \max_{\x\in\Dc} \frac{e^{S\norm{\x}_2}}{1+e^{S\norm{\x}_2}+(K-1)e^{-S\norm{\x}_2}}$. Next, we focus on upper/lower bounding $\kappa$, which is more interesting. 

In Equation \ref{eq:kappa_bounds}, we will show that $\kappa$ scales unfavorably with the size of the set $\Dc\times \Theta$. In our regret analysis in the previous sections, it was useful to  assume in Assumption \ref{assum:boundedness} that $\|\x\|_2\leq 1$. This assumption was without loss of generality because of the following. Suppose a general setting with $X=\max_{\x\in\Dc}\|\x\|_2>1$. We can then define an equivalent MNL model with actions $\tilde\x = \x/X$ and new parameter vector $\tilde\thetab=X\thetab$. The new problem satisfies the unit norm constraint on the radius of $\Dc$ and has a new parameter norm bound $\tilde{S}=SX$. For clarity with regards to the goal of this section, it is convenient to keep track of the radius of $\Dc$ (rather than push it in $S$). Thus, we let $X:=\max_{\x\in\Dc}\|\x\|_2$ (possibly large). We will prove that $\kappa$ grows at least exponentially in $SX$, thus it can be very large if the action decision set is large. 


\begin{figure*}
  \centering
  \begin{subfigure}[b]{0.35\textwidth}
         \centering
         \includegraphics[width=\textwidth]{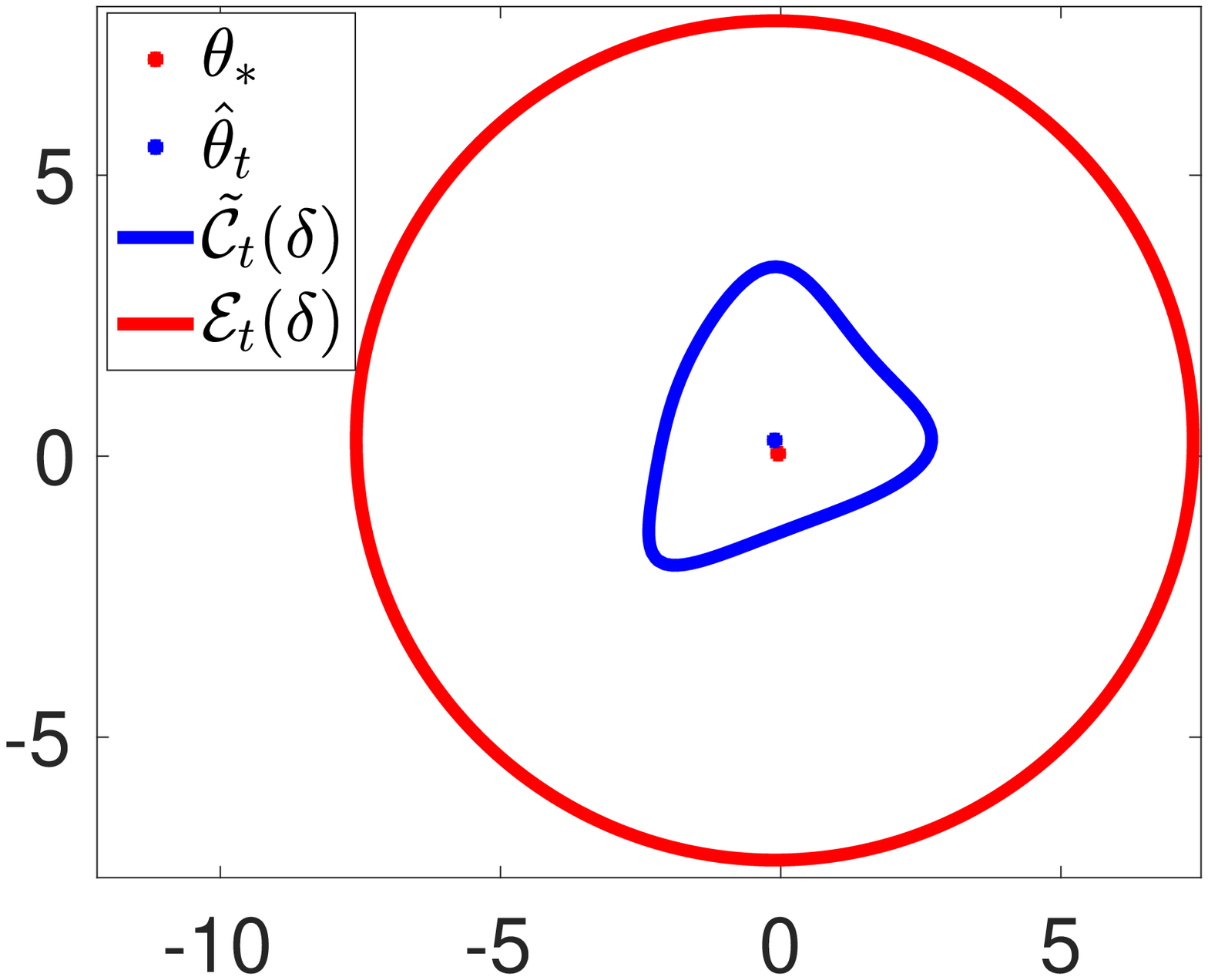}
        \caption{$\kappa = 10$}
          \label{subfig:1}
     \end{subfigure}
     \begin{subfigure}[b]{0.35\textwidth}
         \centering
         \includegraphics[width=\textwidth]{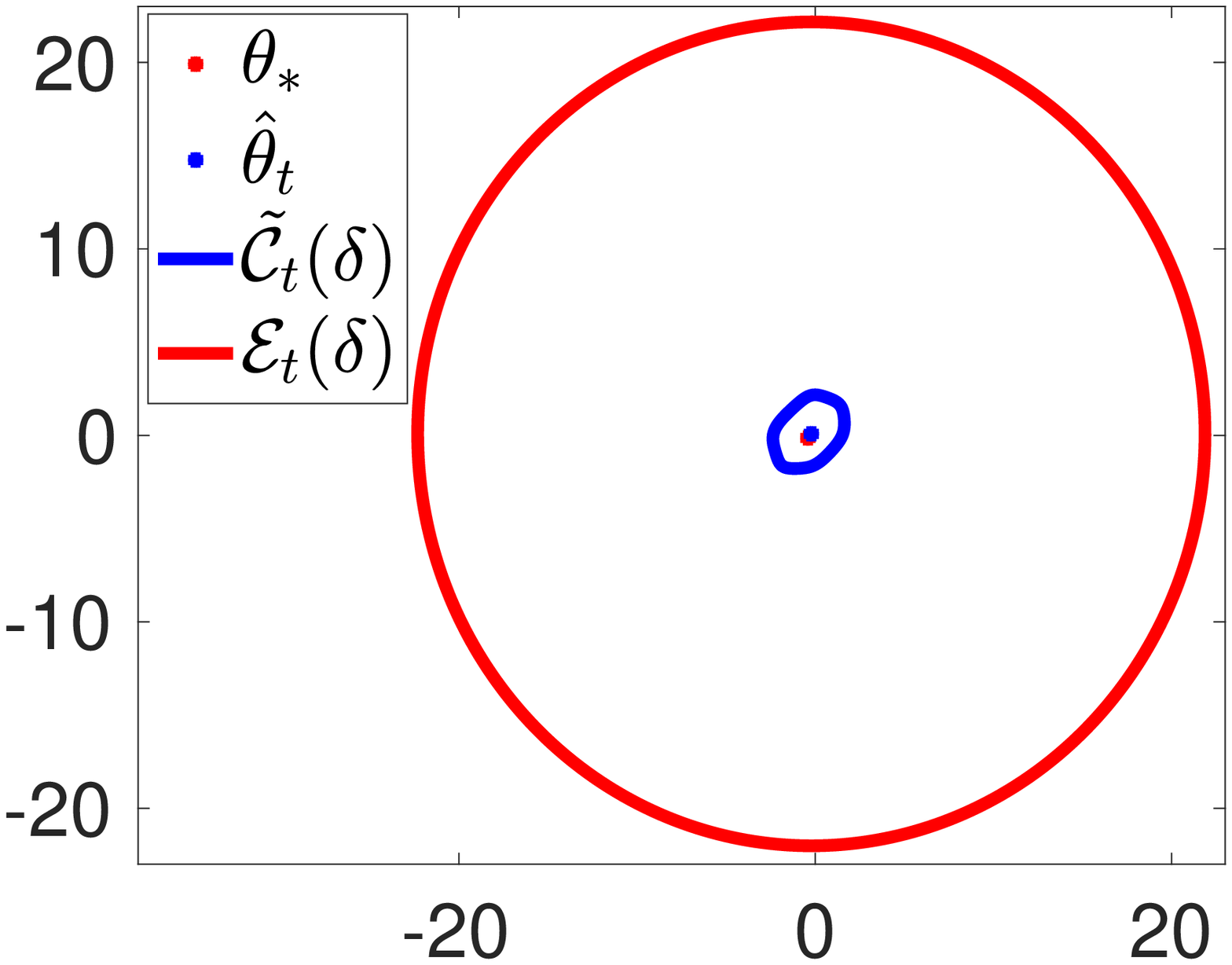}
       \caption{$\kappa=50$}
          \label{subfig:2}
     \end{subfigure}
\caption{Illustration of $\tilde\Cc_t(\delta)$ and $\Ec_t(\delta)$ at round $t=500$ for different values of $\kappa$. Note the ``local dependence" of $\tilde\Cc_t(\delta)$ which results in a non-ellipsoidal shape.}
  \label{fig:confcomp}
\end{figure*}


In order to bound $\kappa$, we identify and take advantage of the following key property of the matrix $\A(\x,\thetab)$ stated as Lemma \ref{lemm:diagonally}. Recall that a matrix $\A\in\mathbb{R}^{K\times K}$ is \emph{strictly diagonally dominant} if each of its diagonal entries is greater than the sum of absolute values of all other entries in the corresponding row/column. We also need the definition of an \emph{$M$-matrix}: A matrix $\A$ is an $M$-matrix if all its off-diagonal entries are non-positive and the real parts of its eigenvalues are non-negative \cite{berman1994nonnegative}. 

\begin{lemma}\label{lemm:diagonally}
For any $\x\in\mathbb{R}^d$ and $\thetab\in\mathbb{R}^{Kd}$, the matrix  $\A(\x,\thetab)$ in \eqref{eq:A} is a strictly diagonally dominant $M$-matrix.
\end{lemma}

This key observation (see Appendix~ \ref{sec:usefullemmas} for the proof) allows us to use Theorem~1.1 in \cite{tian2010inequalities} that provides upper and lower bounds on the minimum eigenvalue of a strictly diagonally dominant $M$-matrix. Specifically, we find that for all all $\x$ and $\thetab$: 
\begin{align}\label{eq:upperandlowerboundsonmineig}
   \min_{i\in[K]}\sum_{j=1}^K\A(\x,\thetab)_{ij}\leq\lamin(\A(\x,\thetab))\leq \max_{i\in[K]}\sum_{j=1}^K\A(\x,\thetab)_{ij}.
  \end{align}
Starting from this and setting $X:=\max_{\x\in\Dc}\norm{\x}_2$ we are able to show the following lower/upper bounds (see Appendix~\ref{sec:kappa_bounds}):
\begin{align}\label{eq:kappa_bounds}
{e^{\frac{SX}{\sqrt{K}}}}(1+Ke^{-\frac{SX}{\sqrt{K}}})^2\leq\kappa\leq e^{SX}{\left(1+Ke^{SX}\right)^2}{}.
\end{align}


 
 \begin{figure*}[t]
\centering
\begin{subfigure}{2in}
\centering
\begin{tikzpicture}
\node at (0,0) {\includegraphics[scale=0.25]{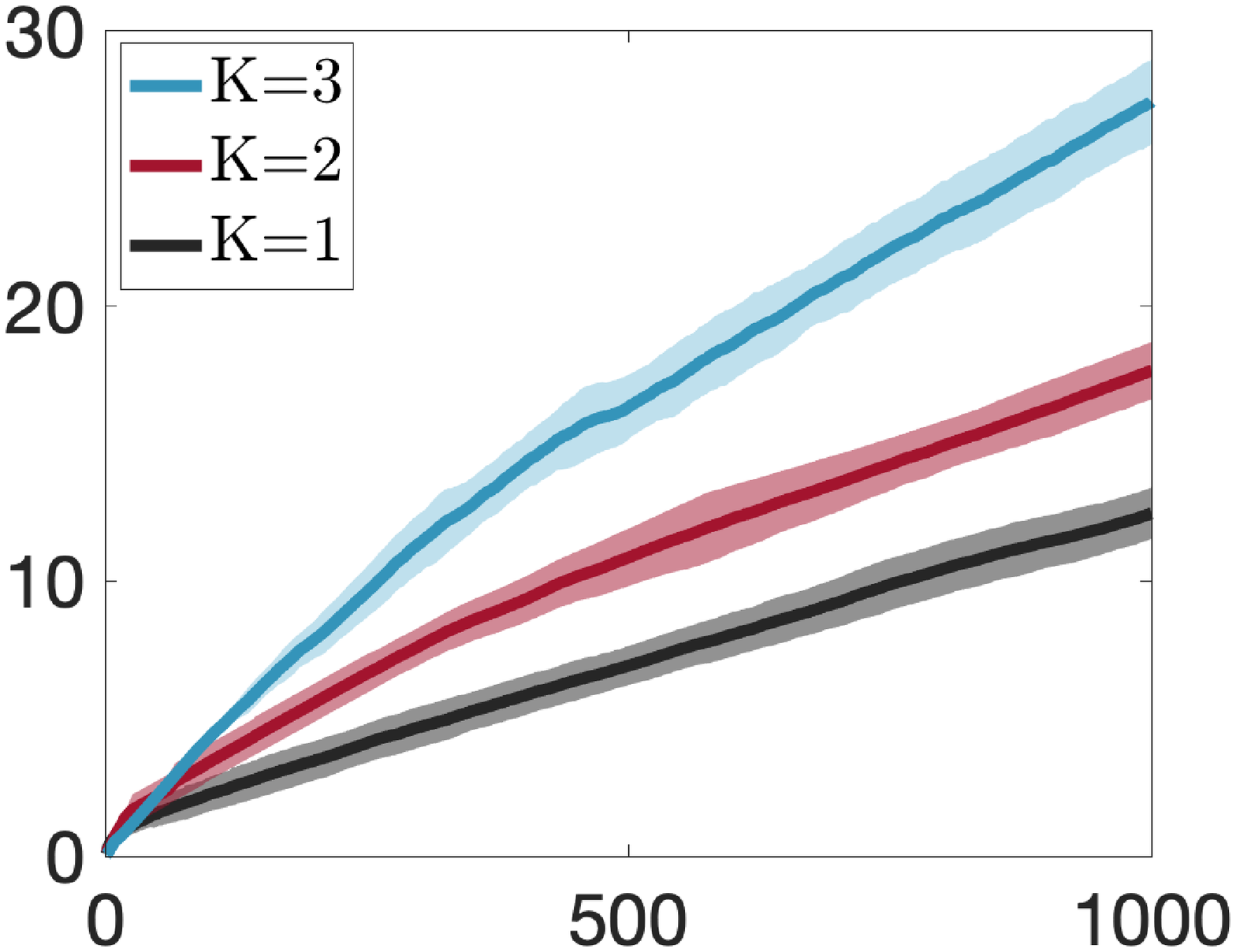}};
\node at (-2.7,0) [rotate=90,scale=0.9]{Regret, $R_t$};
\node at (0,-2.16) [scale=0.9]{Iteration, $t$};
\end{tikzpicture}
\caption{MNL-UCB, $\kappa=30$}
\label{fig:regretdifferentKdifferentK}
\end{subfigure}
\centering
\begin{subfigure}{2in}
\centering
\begin{tikzpicture}
\node at (0,0) {\includegraphics[scale=0.25]{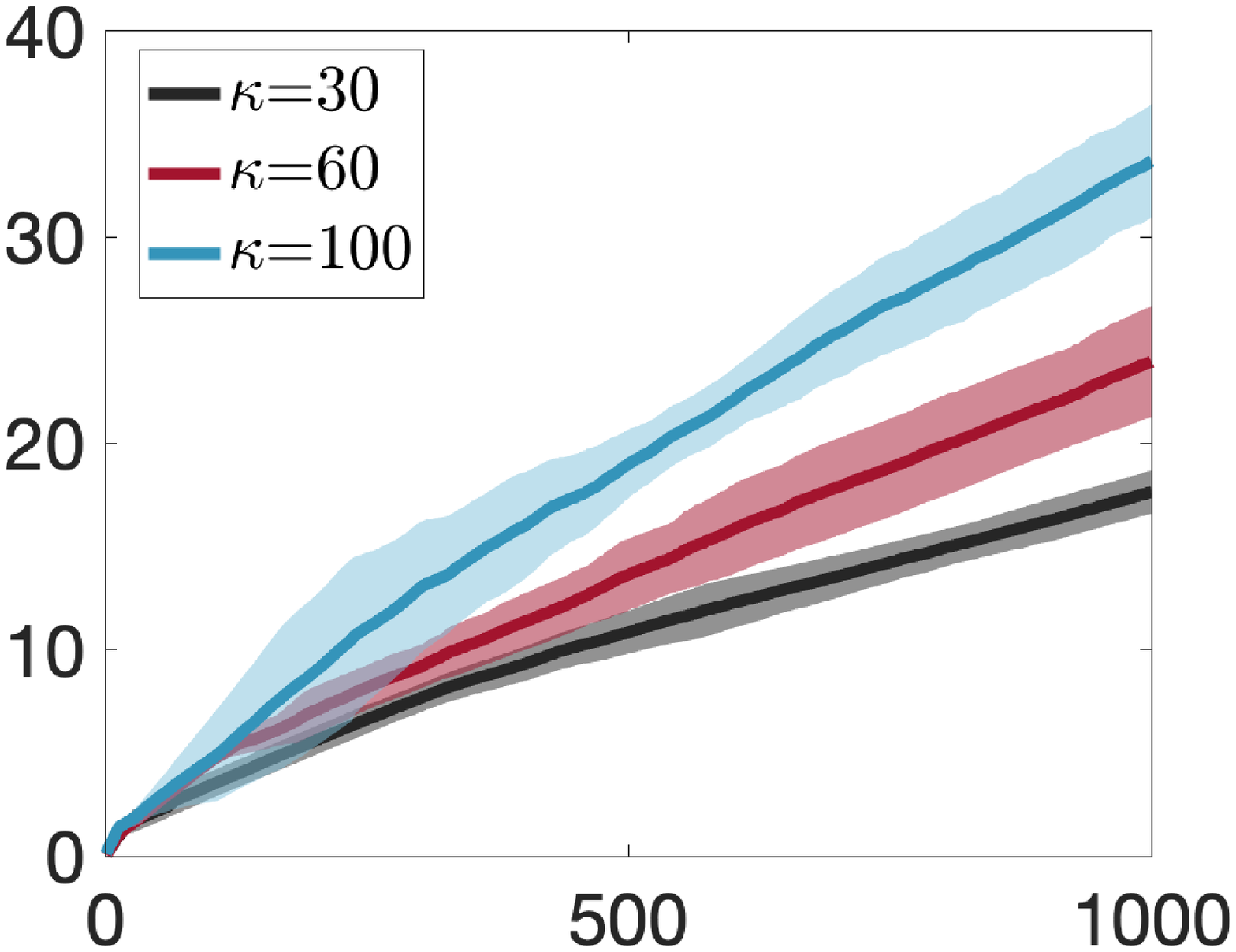}};
\node at (0,-2.16) [scale=0.9]{Iteration, $t$};
\end{tikzpicture}
\caption{MNL-UCB, $K=2$}
\label{fig:regretdifferentKdifferentkappa}
\end{subfigure}
\centering
\begin{subfigure}{2in}
\centering
\begin{tikzpicture}
\node at (0,0) {\includegraphics[scale=0.25]{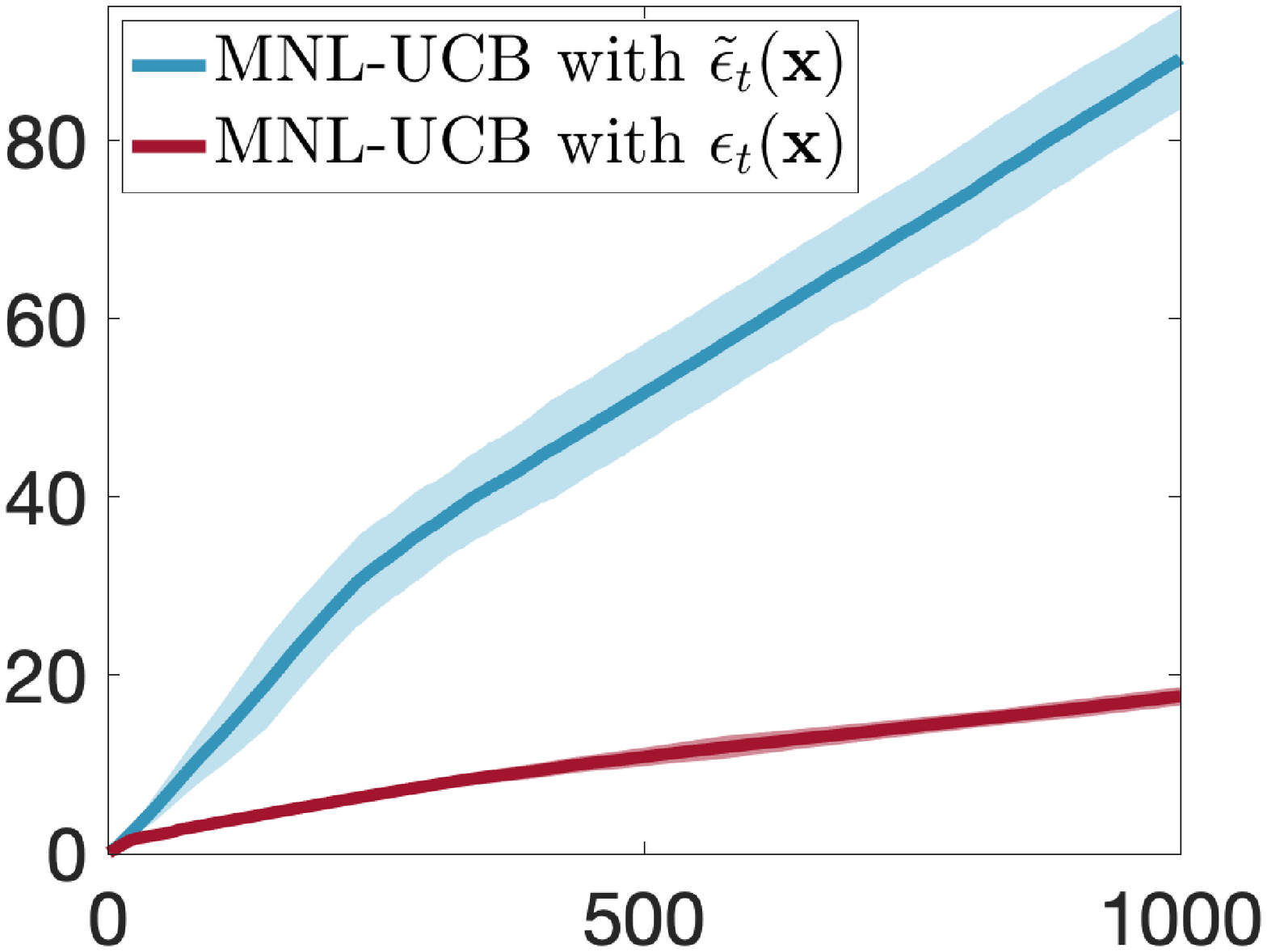}};
\node at (0,-2.16) [scale=0.9]{Iteration, $t$};
\end{tikzpicture}
\caption{$\kappa = 30, K=2$}
\label{fig:errorbound}
\end{subfigure}
\caption{The shaded regions show standard deviation around the average over 20 problem realizations. See text for detailed description.}
\label{fig:regg}
\end{figure*}

\section{Experiments}\label{sec:sim}

We present numerical simulations \footnote{All the experiments are implemented in Matlab on a 2020 MacBook Pro with 32GB of RAM.} to complement and confirm our theoretical findings. 
In all experiments, we used the upper bound on $\kappa$ in \eqref{eq:kappa_bounds} to compute the exploration bonus $\epsilon_t(\x)$. 

{\bf The ``local dependence" of the confidence set $\Cc_t(\delta)$.}
We highlight the superiority of using $\epsilon_t(\x)$ vs $\tilde\epsilon_t(\x)$ introduced in \eqref{eq:Deltaaa} and \eqref{eq:badexplorationbonus}, respectively. In Figure \ref{fig:confcomp}, we give an illustration of how $\Cc_t(\delta)$ compares to $\Ec_t(\delta)$ from which $\epsilon_t(\x)$ and $\tilde\epsilon_t(\x)$ are derived for different values of $\kappa$. In this figure, instead of $\Cc_t(\delta)$, we considered a slightly \emph{larger} confidence set $\tilde\Cc_t(\delta)$ in the more familiar (from linear bandits) form of
\begin{align}
    \tilde\Cc_t(\delta):=\Big\{\thetab \in \Theta: \norm{\thetab-\hat {\boldsymbol\theta}_{t}}_{\Hb^{-1}_{t}(\thetab)} \leq (2+4S)\beta_t(\delta)\Big\},\nn
\end{align}
where $\thetab$ is a ``good" estimate
of $\thetab_\ast$ based on the weighted norm of $\thetab-\hat\thetab_t$ rather than $\g_t(\thetab)-\g_t(\hat\thetab_t)$ in \eqref{eq:thetaconfidenceset}.
Technical details proving $\Cc_t(\delta)\subseteq\tilde\Cc_t(\delta)$ at all rounds $t\in[T]$ are deferred to Appendix \ref{sec:usefullemmas}.
For the sake of visualization, we find it instructive to depict the confidence sets in a 1-dimensional space for a realization with $K=2$. 
The curves of $\tilde\Cc_t(\delta)$ effectively capture the local dependence of $\Hb_t(\thetab)$ on $\thetab$. As a result, unlike $\Ec_t(\delta)$, the confidence set $\tilde\Cc_t(\delta)$ is not an ellipsoid and estimators in different directions are penalized in different ways. Furthermore, the appearance of $\kappa$ in the radius of the confidence ellipsoid $\Ec_t(\delta)$ (see \eqref{eq:badconfidenceset}), which is caused by the use of the matrix $\tilde\Vb_t$ as a global bound on $\Gb_t$ (see Appendix \ref{sec:usefullemmas} for details), results in a larger confidence set compared to $\tilde\Cc_t(\delta)$ that employs the local metric $\Hb_t(\thetab)$. This is also consistent with the observation that as $\kappa$ grows, the difference between the diameters of $\tilde\Cc_t(\delta)$ and $\Ec_t(\delta)$ increases.

{\bf MNL-UCB performance evaluations.}
We evaluate the performance of MNL-UCB on synthetic data. All the results shown depict averages over 20 realizations, for which we have chosen $\delta=0.01$, $d = 2$, and $T=1000$. We considered time-independent decision sets $\Dc$ of 20 arms in $\mathbb{R}^2$ and the reward vector $\boldsymbol{\rho}=[1,\ldots,K]^T$.  Moreover, each arm of the decision set and $\bar\thetab_{\ast i}$ are drawn from $\mathcal{N}(0,I_{d})$ and $\mathcal{N}(0,I_{d}/K)$, respectively. The normalization of the latter by $K$ is so that it guarantees that the problem's signal-to-noise ratio $\norm{\thetab_\ast}_2$ does not change with varying $K$. 
Figure \ref{fig:regretdifferentKdifferentK} depicts the average regret of MNL-UCB for problem settings with different values of $K=1,2$ and $3$. The plot verifies that larger $K$ leads to larger regret and seems to match the proved scaling of the regret bound of MNL-UCB  as $\Oc(K)$ with respect to $K$ .

Figure \ref{fig:regretdifferentKdifferentkappa} showcases the average regret of MNL-UCB for problem settings with fixed $K=2$ and different values of $\kappa=30,60$ and $100$ (upper bounds computed using \eqref{eq:kappa_bounds}). Observe that larger $\kappa$ leads to larger regret. This is consistent with our theoretical findings on the impacts of $\kappa$ on the algorithm's performance. 


Finally, Figure \ref{fig:errorbound} emphasizes the value of using the exploration bonus $\epsilon_t(\x)$ in \eqref{eq:Deltaaa} compared to $\tilde\epsilon_t(\x)$ introduced in \eqref{eq:badexplorationbonus} in the UCB decision making step. In this figure, we fixed $K=2$ and the average regret curves are associated with a problem setting with $\kappa=30$. A comparison between regret curves further confirms the worse regret performance of MNL-UCB when it exploits $\tilde\thetab_t$ and $\tilde\epsilon_t(\x)$ rather than $\thetab_t$ and $\epsilon_t(\x)$ in the UCB decision rule at Line \ref{line:ucb} of the Algorithm \ref{alg:MNL-UCB}.

\section{Conclusion}

For the MNL regression bandit problem, we developed MNL-UCB and showed a regret $\Otilde(Kd\sqrt{\kappa}\sqrt{T})$ that scales favorably with the critical problem-dependent parameter $\kappa$ and optimally with respect to the number of actions $K$. It is interesting to study the efficacy of Thompson sampling-based algorithms for this new problem. Also, extending our results to other multi-index models is yet another important future direction.

\section*{Acknowledgments}
This work is supported by the National Science Foundation under Grant Number (1934641).

\newpage
\bibliographystyle{apalike}
\bibliography{main}
\newpage
\appendix
\onecolumn

\section*{Notation}

Before presenting the proofs of our results, we recall some key notation for the reader's convenience. First, recall that $\thetab=\begin{bmatrix}\bar \thetab_{1}\\ \bar \thetab_{2}\\ \vdots \\ \bar \thetab_{K}\end{bmatrix}\in\mathbb{R}^{Kd}$, where $\bar\thetab_i\in\mathbb{R}^d$.

\paragraph{MNL model.}
\begin{align*}
    {\rm lse}(\s):=\log(1+\sum_{i=1}^{K}\exp(\s_i)).
\end{align*}
\begin{align*}
    \z(\x,\thetab)=\nabla{\rm lse}\left([\bar\thetab_1^T\x,\bar\thetab_2^T\x,\ldots,\bar\thetab_K^T\x]^T\right) = \frac{1}{1+\sum_{j=1}^K\exp(\bar \thetab_{j}^T\x)}\begin{bmatrix}
    \exp(\bar \thetab_{1}^T\x) \\\exp(\bar \thetab_{2}^T\x)\\ \vdots\\ \exp(\bar \thetab_{K}^T\x)
    \end{bmatrix}.
\end{align*}

\begin{align*}
    \A(\x,\thetab) = \nabla^2{\rm lse}\left( [\bar\thetab_1^T\x,\bar\thetab_2^T\x,\ldots,\bar\thetab_K^T\x]^T\right) = \rm{diag}( \z(\x,\thetab)) -  \z(\x,\thetab) \z(\x,\thetab)^T.
\end{align*}

\paragraph{Confidence set.}

\begin{align*}
     \Hb_t(\thetab):= \la I_{Kd}+\sum_{s=1}^{t-1}\A(\x_s,\thetab)\otimes \x_s\x_s^T
\end{align*}

\begin{align*}
 \g_t(\thetab):=\la \thetab+\sum_{s=1}^{t-1}\z(\x_s,\thetab)\otimes\x_s.
\end{align*}
Recall the definition of $\nabla_\thetab\Lc_t^\la(\thetab)$ in \eqref{eq:gradoflogloss}, which we repeat here for the reader's convenience:
\begin{align}\notag
    \nabla_\thetab\Lc_t^\la(\thetab):= \sum_{s=1}^{t-1} [\m_{s}-\z(\x_s,\thetab)]\otimes \x_s-\la \thetab.
\end{align}
By definition, $\hat\thetab_t$ satisfies  $\nabla_\thetab\Lc_t^\la(\hat\thetab_t)=\mathbf{0}$. 

\begin{align*}
\Cc_{t}(\delta):=\left\{\thetab \in \Theta: \norm{{\g_t(\thetab)-\g_t(\hat {\boldsymbol\theta}_{t}})}_{\Hb^{-1}_{t}(\thetab)} \leq \beta_t(\delta)\right\}.
\end{align*}

\section{Proof of Theorem \ref{thm:confidenceset}}\label{sec:confidencesetproof}

In this section, we present the proof of Theorem \ref{thm:confidenceset}.

To see what is necessary to prove Theorem \ref{thm:confidenceset}, we start with the following standard argument to analyze  $\norm{{\g_t(\thetab_\ast)-\g_t(\hat {\boldsymbol\theta}_{t}})}_{\Hb^{-1}_{t}(\thetab^\ast)}$. 
By definition, $\hat\thetab_t$ satisfies  $\nabla_\thetab\Lc_t^\la(\hat\thetab_t)=\mathbf{0}$. Consequently, it holds that
\begin{align}
    \g_t(\thetab_\ast)-\g_t(\hat {\boldsymbol\theta}_{t})=\la\thetab_\ast+\s_t,
\end{align}
where we defined
\begin{align}
    \epsilonb_s:=\z(\x_s,\thetab_\ast)-\m_s \quad\text{and}\quad  \s_t:=\sum_{s=1}^{t-1}\epsilonb_s\otimes\x_s.
\end{align}

Assumption \ref{assum:boundedness} and simple linear algebra then imply that
\begin{align}\label{eq:firstinthmconf}
     \norm{\g_t(\thetab_\ast)-\g_t(\hat {\boldsymbol\theta}_{t})}_{\Hb^{-1}_{t}(\thetab^\ast)}\leq \norm{\s_t}_{\Hb^{-1}_{t}(\thetab^\ast)}+\sqrt{\la}S.
\end{align}
Therefore, establishing an appropriate confidence set $\Cc_t(\delta)$ requires us controlling  $\norm{\s_t}_{\Hb^{-1}_{t}(\thetab^\ast)}$. We do this in several steps organized in a series of key lemmas presented next. \ref{thm:confidenceset}.
\subsection{Necessary lemmas}

The following is a  matrix version of Lemma 7 in \cite{faury2020improved} that uses a rather standard argument to bound the moment generating function of a random vector with bounded components in terms of its covariance matrix.
\begin{lemma}\label{lemm:firstnecessary}
Let $\epsilonb\in\mathbb{R}^K$ be a zero-mean random vector with covariance matrix $\mathbf{C}$. Then, for any vector $\etab\in\mathbb{R}^K$ such that $|\epsilonb^T\etab|\leq 1$, we have
\begin{align}
  \mathbb{E}\left[\exp(\epsilonb^T\etab)\right]\leq \exp\left(\etab^T\mathbf{C}\etab\right).
\end{align}
\end{lemma}
\begin{proof}
Starting with Taylor expansion of the exponential, we have the following chain of inequalities
\begin{align}
    \mathbb{E}\left[\exp(\epsilonb^T\etab)\right] &= 1+\mathbb{E}\left[\epsilonb^T\etab\right]+\sum_{k=2}^\infty \frac{\mathbb{E}\left[(\epsilonb^T\etab)^k\right]}{k!}\nn\\
    &= 1+\sum_{k=2}^\infty \frac{\mathbb{E}\left[(\epsilonb^T\etab)^k\right]}{k!}\tag{$\mathbb{E}[\epsilonb]=\mathbf{0}$}\\
    &\leq 1+\sum_{k=2}^\infty \frac{\mathbb{E}\left[(\epsilonb^T\etab)^2\right]}{k!}\tag{$|\epsilonb^T\etab|\leq 1$}\\
    &= 1+ \left(\etab^T\mathbf{C}\etab\right)\sum_{k=2}^\infty\frac{1}{k!}\nn \leq 1+\etab^T\mathbf{C}\etab,\nn
    \\
    & \leq\exp\left(\etab^T\mathbf{C}\etab\right).\tag{$1+x\leq \exp(x), \forall x\in\mathbb{R}$}\nn
\end{align}
This completes the proof.
\end{proof}

Hereafter, we will denote the $p$-dimensional unit ball by $\Bc_2(p): = \{\nub\in\mathbb{R}^p:\norm{\nub}_2\leq 1\}$.
\begin{lemma}
For all $\xib\in\frac{1}{\sqrt{K}}\Bc_2(Kd)$, let $M_0(\xib)=1$ and for $t>1$
\begin{align}
    M_t(\xib)=\exp(\xib^T\s_t-\norm{\xib}_{\bar\Hb_t(\thetab_\ast)}^2),
\end{align}
where $\bar\Hb_t(\thetab_\ast):=\sum_{s=1}^{t-1}\A(\x_s,\thetab_\ast)\otimes \x_s\x_s^T$. Then $\{M_t(\xib)\}_{t=1}^\infty$ is a non-negative super-martingle.
\end{lemma}
\begin{proof}
For $t>1$, let $\etab_t(\xib) = [\bar\xib_1^T\x_t,\ldots,\bar\xib_K^T\x_t]$. For all $\xib\in\frac{1}{\sqrt{K}}\Bc_2(Kd)$ and $t > 1$, we have
\begin{align}
    \mathbb{E}\left[M_t(\xib)|\Fc_{t-1}\right]&= \mathbb{E}\left[\exp(\xib^T\s_t)|\Fc_{t-1}\right]\exp\left(-\norm{\xib}_{\bar\Hb_t(\thetab_\ast)}^2\right)\nn\\
    &=\mathbb{E}\left[\exp(\xib^T[\epsilonb_{t-1}\otimes\x_{t-1}])|\Fc_{t-1}\right]\exp\left(\xib^T\s_{t-1}-\norm{\xib}_{\bar\Hb_t(\thetab_\ast)}^2\right)\nn\\
    &=\mathbb{E}\left[\exp\left(\epsilonb_{t-1}^T\etab_{t-1}(\xib)\right)|\Fc_{t-1}\right]\exp\left(\xib^T\s_{t-1}-\norm{\xib}_{\bar\Hb_t(\thetab_\ast)}^2\right)\label{eq:firsteq}.
\end{align}
Note that $\mathbb{E}\left[\epsilonb_t|\Fc_{t-1}\right]=\mathbf{0}$ and $\mathbb{E}\left[\epsilonb_t\epsilonb_t^T|\Fc_{t-1}\right]=\A(\x_t,\thetab_\ast)$ for all $t>1$. Furthermore, $|\epsilonb_{t-1}^T\etab_{t-1}(\xib)|\leq 1$ because
\begin{align}
    |\epsilonb_{t-1}^T\etab_{t-1}(\xib)|&\leq \norm{\epsilonb_{t-1}}_2\norm{\etab_{t-1}}_2\tag{Cauchy-Schwarz}\\
    &\leq \sqrt{K}\sqrt{\sum_{i=1}^K \left(\bar\xib_i^T\x_t\right)^2} \tag{$|[\epsilonb_{t-1}]_i|\leq 1$}\\
    & = \sqrt{K}\sqrt{\x_t^T\left(\sum_{i=1}^K\bar\xib_i\bar\xib_i^T\right)\x_t}\nn\\
    &\leq \norm{\x_t}_2\sqrt{K\lamax\left(\sum_{i=1}^K\bar\xib_i\bar\xib_i^T\right)}\nn\\
    &\leq \sqrt{K{\rm Tr}\left(\sum_{i=1}^K\bar\xib_i\bar\xib_i^T\right)}\tag{Assumption \ref{assum:boundedness}}\\
    &=\sqrt{K\sum_{i=1}^K\norm{\bar\xib_i}_2^2}\nn\\
    &=\sqrt{K\norm{\xib}_2^2}\nn\\
    &\leq 1 \tag{$\norm{\xib}_2\leq \frac{1}{\sqrt{K}}$}.
\end{align}

Thus, using Lemma \ref{lemm:firstnecessary}, we conclude that
\begin{align}
    \mathbb{E}\left[\exp\left(\epsilonb_{t-1}^T\etab_{t-1}(\xib)\right)|\Fc_{t-1}\right]&\leq 1+\etab_{t-1}^T(\xib)\A(\x_{t-1},\thetab_\ast)\etab_{t-1}(\xib)\nn\\
   &\leq\exp\left(\xib^T\left(\A(\x_{t-1},\thetab_\ast)\otimes\left[\x_{t-1}\x_{t-1}^T\right]\right)\xib\right), \label{eq:secondeq}
\end{align}
Combining \eqref{eq:firsteq} and \eqref{eq:secondeq}, we have
\begin{align}
    \mathbb{E}\left[M_t(\xib)|\Fc_{t-1}\right]\leq \exp\left(\xib^T\s_{t-1}-\norm{\xib}_{\bar\Hb_t(\thetab_\ast)}^2+\xib^T\left(\A(\x_{t-1},\thetab_\ast)\otimes\left[\x_{t-1}\x_{t-1}^T\right]\right)\xib\right)=M_{t-1}(\xib),
\end{align}
as desired.
\end{proof}

Now, let $h(\xib)$ be a probability measure with support on $\frac{1}{\sqrt{K}}\Bc_2(Kd)$ and define
\begin{align}
    \bar M_t=\int_{\xib\in\frac{1}{\sqrt{K}}\Bc_2(Kd)}M_t(\xib)dh(\xib)=
    \int_{\xib\in\frac{1}{\sqrt{K}}\Bc_2(Kd)}\exp(\xib^T\s_t-\norm{\xib}_{\bar\Hb_t(\thetab_\ast)}^2)\,dh(\xib)
    .\label{eq:M_t}
\end{align}
By Lemma~20.3 in \cite{lattimore2018bandit} $\bar M_t$ is also a non-negative super-martingale
and $\mathbb {E}\left[\bar M_0\right]=1$. Let $\tau$ be the stopping point with respect to the filtration $\{\Fc_t\}_{t=0}^\infty$. By Lemma~8 \cite{abbasi2011improved}, $\bar M_\tau$ is well-defined and $\mathbb {E}\left[\bar M_\tau\right]\leq 1$. Thanks to Markov inequality, for any $\delta\in(0,1)$, we have
\begin{align}
    \mathbb{P}(\bar M_\tau\geq\frac{1}{\delta})\leq \delta \mathbb{E}[\bar M_\tau] \leq \delta.\label{eq:first1}
\end{align}

\begin{lemma}\label{lemm:main1}
Let $h(\xib)$ be the density of an isotropic normal distribution with precision matrix (i.e. inverse covariance matrix) $2\la I_{Kd}$ that is  truncated on
$\frac{1}{\sqrt{K}}\Bc_2(Kd)$ and $g(\xib)$ be the density of the normal distribution with precision matrix $2\Hb_t(\thetab_\ast)$ that is truncated on the ball
$\frac{1}{2\sqrt{K}}\Bc_2(Kd)$. Denote the normalization constant of $h$ and $g$ by $N(h)$ and $N(g)$, respectively. Then
\begin{align}
    \mathbb{P}\left(\norm{\s_t}_{\Hb_t^{-1}(\thetab_\ast)}\geq \frac{\sqrt{\la/K}}{2}+2\sqrt{K/\la}\log\left(\frac{N(h)}{\delta N(g)}\right)\right)\leq \delta.
\end{align}
\end{lemma}
\begin{proof}
We prove this lemma using similar techniques employed in \cite{abbasi2011improved,faury2020improved}. By the definition of $\bar M_t$ in \eqref{eq:M_t} and recalling $\Hb_t(\thetab_\ast) = \lambda\mathbf{I}_{Kd}+\bar \Hb_t(\thetab_\ast)$, we have
\begin{align}
    \bar M_t = \frac{1}{N(h)}\int_{\xib\in\frac{1}{\sqrt{K}}\Bc_2(Kd)}\exp(\xib^T\s_t-\norm{\xib}_{\Hb_t(\thetab_\ast)}^2)d\xib.
\end{align}
Now, let $f(\xib)=\xib^T\s_t-\norm{\xib}_{\Hb_t(\thetab_\ast)}^2$ and $\xib_\ast=\argmax_{\norm{\xib}_2\leq \frac{1}{2\sqrt{K}}}f(\xib)$. Since
\begin{align}
   f(\xib)=f(\xib_\ast)+(\xib-\xib_\ast)^T\nabla f(\xib_\ast)-(\xib-\xib_\ast)^T\Hb_t(\thetab_\ast)(\xib-\xib_\ast),
\end{align}
we have
\begin{align}
    \bar M_t &= \frac{\exp(f(\xib_\ast))}{N(h)}\int_{\mathbb{R}^{Kd}}\mathbbm{1}\left\{\norm{\xib}_2\leq\frac{1}{\sqrt{K}}\right\}\exp\left((\xib-\xib_\ast)^T\nabla f(\xib_\ast)-\norm{\xib-\xib_\ast}^2_{\Hb_t(\thetab_\ast)}\right)d\xib\nn\\
    &= \frac{\exp(f(\xib_\ast))}{N(h)}\int_{\mathbb{R}^{Kd}}\mathbbm{1}\left\{\norm{\xib+\xib_\ast}_2\leq\frac{1}{\sqrt{K}}\right\}\exp\left(\xib^T\nabla f(\xib_\ast)-\norm{\xib}^2_{\Hb_t(\thetab_\ast)}\right)d\xib\tag{change of variable}\\
    &\geq \frac{\exp(f(\xib_\ast))}{N(h)}\int_{\mathbb{R}^{Kd}}\mathbbm{1}\left\{\norm{\xib}_2\leq\frac{1}{2\sqrt{K}}\right\}\exp\left(\xib^T\nabla f(\xib_\ast)-\norm{\xib}^2_{\Hb_t(\thetab_\ast)}\right)d\xib\tag{$\norm{\xib_\ast}_2\leq \frac{1}{2\sqrt{K}}$}\\
    &= \exp(f(\xib_\ast))\frac{N(g)}{N(h)}\mathbb{E}_g\left[\exp\left(\xib^T\nabla f(\xib_\ast)\right)\right]\tag{definition of $g(\xib)$}\\
    &\geq \exp(f(\xib_\ast))\frac{N(g)}{N(h)}\exp\left(\mathbb{E}_g\left[\xib^T\nabla f(\xib_\ast)\right]\right)\tag{Jensen's inequality}\\
   &= \exp(f(\xib_\ast))\frac{N(g)}{N(h)}\label{eq:second1},
\end{align}
where in the last equality we used the fact that $\mathbb{E}_g[\xib]=\mathbf{0}$. Now, let $\xib_0 = \frac{\Hb_t^{-1}(\thetab_\ast)\s_t}{\norm{\s_t}_{\Hb_t^{-1}(\thetab_\ast)}}\times \frac{\sqrt{\la/K}}{2}$. Hence
\begin{align}
    \norm{\xib_0}_2 = \frac{\sqrt{\la/K}}{2}\cdot\frac{\norm{\s_t}_{\Hb_t^{-2}(\thetab_\ast)}}{\norm{\s_t}_{\Hb_t^{-1}(\thetab_\ast)}}\leq \frac{\sqrt{\la/K}}{2}\cdot \frac{1}{\sqrt{\la}} = \frac{1}{2\sqrt{K}}.
\end{align}
Combining \eqref{eq:first1} and \eqref{eq:second1}, we conclude that
\begin{align}
    \mathbb{P}\left(\xib_0^T\s_t-\norm{\xib_0}_{\Hb_t(\thetab_\ast)}^2\geq \log(\frac{1}{\delta})+\log\left(\frac{N(h)}{N(g)}\right)\right) &\leq \mathbb{P}\left(\max_{\norm{\xib}_2\leq\frac{1}{2\sqrt{K}}}\xib^T\s_t-\norm{\xib}_{\Hb_t(\thetab_\ast)}^2\geq \log(\frac{1}{\delta})+\log\left(\frac{N(h)}{N(g)}\right)\right) \nn\\
    &= \mathbb{P}\left(\exp\left(f(\xib_\ast)\right)\frac{N(g)}{N(h)}\geq \frac{1}{\delta}\right) \nn\\
    &\leq \mathbb{P}\left(\bar M_t\geq \frac{1}{\delta}\right) \leq \delta.
\end{align}
Thus, by the definition of $\xib_0$, we have
\begin{align}
    \mathbb{P}\left(\frac{\sqrt{\la/K}}{2}\norm{\s_t}_{\Hb_t^{-1}(\thetab_\ast)}\geq\frac{\la}{4K}+\log\left(\frac{N(h)}{\delta N(g)}\right) \right) =\mathbb{P}\left(\norm{\s_t}_{\Hb_t^{-1}(\thetab_\ast)}\geq \frac{\sqrt{\la/K}}{2}+2\sqrt{K/\la}\log\left(\frac{N(h)}{\delta N(g)}\right)\right)\leq \delta,
\end{align}
as desired.
\end{proof}

We use the following lemma borrowed from \cite{faury2020improved} to complete the proof of Theorem \ref{thm:confidenceset}.

\begin{lemma}[Lemma 6 in \cite{faury2020improved}]\label{lemm:main2}
It holds that
\begin{align}
    \log\left(\frac{N(h)}{\delta N(g)}\right)\leq \log\left(\frac{\det\left(\Hb_t^{1/2}(\thetab_\ast)\right)}{\la^{Kd/2}}\right)+Kd\log(2/\delta)
\end{align}
\end{lemma}

\subsection{Completing the proof of Theorem 
\ref{thm:confidenceset}}
We are now ready to complete the proof of Theorem \ref{thm:confidenceset} combining the key results presented above as follows:
\begin{align}
1-\delta&\leq \mathbb{P}\left(\norm{\s_t}_{\Hb^{-1}_{t}(\thetab^\ast)}+\sqrt{\la}S\leq \frac{\sqrt{\la/K}}{2}+2\sqrt{K/\la}\log\left(\frac{N(h)}{\delta N(g)}\right)+\sqrt{\la}S\right) \tag{Lemma \ref{lemm:main1}}\\
&\leq\mathbb{P}\left(\norm{\s_t}_{\Hb^{-1}_{t}(\thetab^\ast)}+\sqrt{\la}S\leq \frac{\sqrt{\la/K}}{2}+2\sqrt{K/\la}\log\left(\frac{\det\left(\Hb_t^{1/2}(\thetab_\ast)\right)}{\la^{Kd/2}}\right)+\frac{2K^{3/2}d}{\sqrt{\la}}\log(2/\delta)+\sqrt{\la}S\right)\tag{Lemma \ref{lemm:main2}}\\
&\leq \mathbb{P}\left(\norm{\g_t(\thetab_\ast)-\g_t(\hat {\boldsymbol\theta}_{t})}_{\Hb^{-1}_{t}(\thetab^\ast)}\leq\frac{\sqrt{\la/K}}{2}+2\sqrt{K/\la}\log\left(\frac{\det\left(\Hb_t^{1/2}(\thetab_\ast)\right)}{\la^{Kd/2}}\right)+\frac{2K^{3/2}d}{\sqrt{\la}}\log(2/\delta)+\sqrt{\la}S\right)\tag{Eqn.~\eqref{eq:firstinthmconf}}\\
&\leq\mathbb{P}\left(\norm{\g_t(\thetab_\ast)-\g_t(\hat {\boldsymbol\theta}_{t})}_{\Hb^{-1}_{t}(\thetab^\ast)}\leq\frac{K^{3/2}d}{\sqrt{\la}}\log\left(1+\frac{t}{d\la}\right)+\frac{\sqrt{\la/K}}{2}+\frac{2K^{3/2}d}{\sqrt{\la}}\log(2/\delta)+\sqrt{\la}S\right),
\end{align}
where the last inequality follows from
\begin{align}
    \det\left(\Hb_t(\thetab_\ast)\right)=\prod_{i=1}^{Kd} \lambda_i(\Hb_t(\thetab_\ast)) &\leq \left({\rm trace}\left(\Hb_t(\thetab_\ast)\right)/Kd\right)^{Kd},
\end{align}
and
\begin{align}
   {\rm trace}\left(\Hb_t(\thetab_\ast)\right)\leq \la Kd+\sum_{s=1}^{t-1}\sum_{i=1}^K\A(\x_s,\thetab_\ast)_{ii}\norm{\x_s}_2^2\leq  \la Kd+Kt\tag{Assumption \ref{assum:boundedness}}.
\end{align}


\section{Proofs of Lemma \ref{lemm:pred} and Theorem \ref{thm:regretbound}}\label{sec:predlemmaproof}
As mentioned in Section \ref{sec:proofsketch}, we organize the proof of Lemma \ref{lemm:pred} in several key lemmas which we prove below

\paragraph{Conditioning on $\boldsymbol\theta_\ast\in\Cc_{t}(\delta),~\forall t\in[T]$.}  
Consider the event 
\begin{align}\label{eq:conditioned}
\Ec:=\{\boldsymbol\theta_\ast\in \Cc_{t}(\delta),~\forall t\in[T]\},
\end{align}
that $\boldsymbol\theta_\ast$ is inside the confidence sets for all rounds $t\in[T]$. By Theorem \ref{thm:confidenceset} the event holds with probability at least $1-\delta$. Onwards, we condition on this event, and make repeated use of the fact that $\boldsymbol\theta_\ast\in\Cc_{t}$ for all $t\in[T]$, without further explicit reference.

\subsection{Necessary lemmas}\label{sec:necessarypredlemmaproof}


Here, we prove Lemmas \ref{lemm:z(1)-z(2)original}, \ref{lemm:g_t(1)-g_t(2)original}, and \ref{lemm:generalselforiginal} of Section \ref{sec:proofsketch}. For the reader's convenience, we repeat their statements as  Lemmas   \ref{lemm:z(1)-z(2)}, \ref{lemm:g_t(1)-g_t(2)}, and \ref{lemm:generalself} below.

\begin{lemma}\label{lemm:z(1)-z(2)}
Recall the definition of the $\rm lse$ function:
$
    {\rm lse}(\s):=\log(1+\sum_{i=1}^{K}\exp(\s_i))
$
and let $\boldsymbol\mu:\mathbb{R}^K\rightarrow \mathbb{R}^K$ be defined as
\begin{align}\label{eq:mu}
    \boldsymbol\mu(\s):= \nabla {\rm lse}(\s)= \frac{1}{1+\sum_{i=1}^K\exp(\s_i)}\begin{bmatrix}\exp(\s_1)\\
    \vdots\\
    \exp(\s_K)
    \end{bmatrix}.
\end{align}
Moreover, for $\x\in\mathbb{R}^d$ and $\thetab_1,\thetab_2\in\mathbb{R}^{Kd}$ define
\begin{align}\label{eq:B(x,1,2)}
    \B(\x,\thetab_1,\thetab_2):=\int_{0}^1\nabla\boldsymbol\mu\left(\begin{bmatrix}v\bar\thetab_{11}^T\x+(1-v)\bar\thetab_{21}^T\x\\
    v\bar\thetab_{12}^T\x+(1-v)\bar\thetab_{22}^T\x\\\vdots\\v\bar\thetab_{1K}^T\x+(1-v)\bar\thetab_{2K}^T\x
    \end{bmatrix}\right)dv = \int_{0}^1\A(\x,v\thetab_1+(1-v)\thetab_2)dv,
\end{align}
where the matrix $\A(\x,\thetab)$ is given in \eqref{eq:A}. 
Then, it holds that
\begin{align}\label{eq:z_t-z_t}
    \z(\x,\thetab_1)-\z(\x,\thetab_2)=\left[\B(\x,\thetab_1,\thetab_2)\otimes \x^T\right] (\thetab_1-\thetab_2).
\end{align}
\end{lemma}

\begin{remark}[Notation]
In the notation of the lemma, notice that our vector function $\z(\x,\thetab)$ defined in \eqref{eq:z_def} can be expressed in terms of $\boldsymbol{\mu}$ as:
$$
\z(\x,\thetab) = \boldsymbol{\mu}([\bar\thetab_1^T\x,\ldots,\bar\thetab_K^T\x]^T).
$$
Similarly, 
$$
\A(\x,\thetab) = \nabla\boldsymbol{\mu}([\bar\thetab_1^T\x,\ldots,\bar\thetab_K^T\x]^T).
$$
\end{remark}

\begin{proof}
According to \emph{mean-value Theorem}, for a differentiable function $\f:\mathbb{R}^P\rightarrow\mathbb{R}^Q$ and differentiable function $\q:\mathbb{R}\rightarrow\mathbb{R}^P$, we have \begin{align}
    \int_{a}^b \nabla\f^T(\q(v))\q^\prime(v)dv= \f(\q(b))-\f(\q(a)).
\end{align}

Now, let $\f(\thetab)=\z(\x,\thetab)$, $a=0$, $b=1$ and $\q(v)=v\thetab_1+(1-v)\thetab_2$. By the mean-value theorem, we have
\begin{align}
  \z(\x,\thetab_1)-\z(\x,\thetab_2)=\left[\B(\x,\thetab_1,\thetab_2)\otimes\x\right]^T (\thetab_1-\thetab_2)=\left[\B(\x,\thetab_1,\thetab_2)\otimes \x^T\right] (\thetab_1-\thetab_2),
\end{align}
where the last equality follows from the fact that $(\A\otimes\B)^T=\A^T\otimes\B^T$ for any matrices $\A$ and $\B$ and matrix $\B(\x,\thetab_1,\thetab_2)\in \mathbb{R}^{K\times K}$ being a symmetric matrix.
\end{proof}

\begin{lemma}\label{lemm:g_t(1)-g_t(2)}
Define
\begin{align}\label{eq:G_t}
    \Gb_t(\thetab_1,\thetab_2):=\la I_{Kd}+\sum_{s=1}^{t-1}\B(\x_s,\thetab_1,\thetab_2)\otimes \x_s\x_s^T.
\end{align}
Then, for any $\thetab_1,\thetab_2\in\mathbb{R}^{Kd}$, we have
\begin{align}\label{eq:g_t-g_t}
    \g_t(\thetab_1)- \g_t(\thetab_2) = \Gb_t(\thetab_1,\thetab_2)(\thetab_1-\thetab_2).
\end{align}
\end{lemma}
\begin{proof}
By the definition of $\g_t(\thetab)$ in \eqref{eq:g_t} we have
\begin{align}
    \g_t(\thetab_1)- \g_t(\thetab_2) &=\la(\thetab_1-\thetab_2)+\sum_{s=1}^{t-1}[\z(\x_s,\thetab_1)-\z(\x_s,\thetab_2)]\otimes \x_s \nn \\
    &=\la(\thetab_1-\thetab_2)+\sum_{s=1}^{t-1}\left[\left[\B(\x_s,\thetab_1,\thetab_2)\otimes \x_s^T\right] (\thetab_1-\thetab_2)\right]\otimes \x_s \tag{Eqn.~\eqref{eq:z_t-z_t}}\\
    &=\la(\thetab_1-\thetab_2)+\sum_{s=1}^{t-1}\left[\left[\B(\x_s,\thetab_1,\thetab_2)\otimes \x_s^T\right]\otimes \x_s\right] (\thetab_1-\thetab_2)\tag{mixed-product property}\\
    &=\la(\thetab_1-\thetab_2)+\sum_{s=1}^{t-1}\left[\B(\x_s,\thetab_1,\thetab_2)\otimes\left[ \x_s^T\otimes \x_s\right]\right] (\thetab_1-\thetab_2)\nn\\
    &=\la(\thetab_1-\thetab_2)+\sum_{s=1}^{t-1}\left[\B(\x_s,\thetab_1,\thetab_2)\otimes\x_s\x_s^T\right] (\thetab_1-\thetab_2)
    =\Gb_t(\thetab_1,\thetab_2)(\thetab_1-\thetab_2)\nn,
\end{align}
as desired.
\end{proof}
\begin{lemma}\label{lemm:positivedefiniteVt}
Define
\begin{align}
    \tilde\Vb_t:=\kappa\la I_{Kd}+\sum_{s=1}^{t-1}\sum_{i=1}^K (\e_i\otimes\x_s)(\e_i\otimes\x_s)^T = I_K\otimes\Vb_t.
\end{align}
Then, for any $\thetab_1,\thetab_2\in\Theta$, we have $$\frac{1}{\kappa}\tilde\Vb_t \preceq \Gb_t(\thetab_1,\thetab_2).$$
\end{lemma}
\begin{proof}
By definition of $\Gb_t(\thetab_1,\thetab_2)$ for any $\thetab_1, \thetab_2\in\Theta$ in \eqref{eq:G_t}, we have
\begin{align}
    \Gb_t(\thetab_1,\thetab_2)&=\la I_{Kd}+\sum_{s=1}^{t-1}\B(\x_s,\thetab_1,\thetab_2)\otimes \x_s\x_s^T\nn\\
    &\succeq \frac{1}{\kappa}\left[\kappa\la I_{Kd}+\sum_{s=1}^{t-1}I_K\otimes \x_s\x_s^T\right]\tag{Assumption \ref{assum:problemdependentconstants}}\\
    &=\frac{1}{\kappa}\tilde\Vb_t.
\end{align}
We derived the inequality above as follows.  Assumption \ref{assum:problemdependentconstants} and definition of $\B(\x,\thetab_1,\thetab_2)$ in \eqref{eq:B(x,1,2)} give $\B(\x,\thetab_1,\thetab_2)\succeq\frac{1}{\kappa}I_K$. Now, from standard spectral properties of the Kronecker product, the eigenvalues of the matrix  $\B(\x,\thetab_1,\thetab_2)\otimes \x_s\x_s^T-\frac{1}{\kappa}I_K\otimes \x_s\x_s^T = (\B(\x,\thetab_1,\thetab_2)-\frac{1}{\kappa}I_K)\otimes \x_s\x_s^T$ are the pairwise products of the individual eigenvalues of $\B(\x,\thetab_1,\thetab_2)-\frac{1}{\kappa}I_K$ and of $\x_s\x_s^T$. The desired inequality then follows since both of the latter two matrices are positive semi-definite.
\end{proof}

\begin{lemma}[Generalized self-concordance]\label{lemm:generalself}
For any $\thetab_1,\thetab_2\in\Theta$, we have
\begin{align}
    (1+2S)^{-1} \Hb_t(\thetab_1)\preceq  \Gb_t(\thetab_1,\thetab_2)\quad\text{and}\quad(1+2S)^{-1} \Hb_t(\thetab_2)\preceq  \Gb_t(\thetab_1,\thetab_2).
\end{align}
\end{lemma}
\begin{proof}
Recall the function ${\rm lse}:\mathbb{R}^K\rightarrow \mathbb{R}$
\begin{align}\label{eq:f}
    {\rm lse}(\s):=\log\left(1+\sum_{i=1}^K\exp(\s_i)\right).
\end{align}
Then its gradient and Hessian are written as
\begin{align}\label{eq:gradandhessianoff}
    \nabla {\rm lse}(\s) = \boldsymbol\mu(\s):=\frac{1}{1+\sum_{i=1}^K\exp(\s_i)}\begin{bmatrix}\exp(\s_1)\\
    \vdots\\
    \exp(\s_K)
    \end{bmatrix}\quad\text{and}\quad \nabla^2 {\rm lse}(\s)=\diag(\s)-\s\s^T.
\end{align}
As shown in Lemma~4 in \cite{tran2015composite}, the function
$\rm lse$ is $(M_{\rm lse},\nu)$-generalized self-concordant with $\nu=1$ and $M_{\rm lse}=\sqrt{6}$. Therefore, according to Corollary~2 in \cite{sun2019generalized}, for any $\x\in\Dc$ and $\thetab_1,\thetab_2\in\Theta$, we have
\begin{align}\label{eq:selfconcordance}
  \frac{1-\exp\left(-d(\x,\thetab_1,\thetab_2)\right)}{d(\x,\thetab_1,\thetab_2)}\nabla^2f\left(\begin{bmatrix}\bar\thetab_{21}^T\x\\
  \bar\thetab_{22}^T\x\\\vdots\\\bar\thetab_{2K}^T\x
    \end{bmatrix}\right) \preceq \int_{0}^1\nabla^2f\left(\begin{bmatrix}v\bar\thetab_{11}^T\x+(1-v)\bar\thetab_{21}^T\x\\
    v\bar\thetab_{12}^T\x+(1-v)\bar\thetab_{22}^T\x\\\vdots\\v\bar\thetab_{1K}^T\x+(1-v)\bar\thetab_{2K}^T\x
    \end{bmatrix}\right)dv,
\end{align}
where $d(\x,\thetab_1,\thetab_2)=\norm{\left[\bar\thetab_{11}^T\x-\bar\thetab_{21}^T\x,\ldots,\bar\thetab_{1K}^T\x-\bar\thetab_{2K}^T\x\right]^T}_2$. Now, by the definition of $\A(\x,\thetab)$ and $\B(\x,\thetab_1,\thetab_2)$ in \eqref{eq:A} and \eqref{eq:B(x,1,2)}, we conclude from \eqref{eq:selfconcordance} that
\begin{align}
    \B(\x,\thetab_1,\thetab_2)&\succeq \left(1+d(\x,\thetab_1,\thetab_2)\right)^{-1}\A(\x,\thetab_2)\tag{$\frac{1-\exp(-t)}{t}\geq (1+t)^{-1}$}\\
    &\succeq \left(1+2S\right)^{-1}\A(\x,\thetab_2)\label{eq:BgreaterthanA},
\end{align}
where the last inequality follows from Assumption \ref{assum:boundedness} and the fact that both $\thetab_1,\thetab_2\in\Theta$. Combining \eqref{eq:BgreaterthanA} and the definition of $\Gb_t(\thetab_1,\thetab_2)$ in \eqref{eq:G_t} for any $\thetab_1,\thetab_2\in\Theta$, we have
\begin{align}
    \Gb_t(\thetab_1,\thetab_2)&=\la I_{Kd}+\sum_{s=1}^{t-1}\B(\x_s,\thetab_1,\thetab_2)\otimes \x_s\x_s^T \nn \\
    &\succeq \left(1+2S\right)^{-1}\left[\la I_{Kd}+\sum_{s=1}^{t-1}\A(\x_s,\thetab_2)\otimes \x_s\x_s^T\right]\nn\\
    &= \left(1+2S\right)^{-1}\Hb_t(\thetab_2)\tag{by definition of $\Hb_t(\thetab_2)$ in Eqn.~\eqref{eq:Hessian}}.
\end{align}
By symmetric roles of $\thetab_1$ and $\thetab_2$ in the definition of $\B(\x,\thetab_1,\thetab_2)$, we can similarly prove that $(1+2S)^{-1} \Hb_t(\thetab_1)\preceq  \Gb_t(\thetab_1,\thetab_2)$.
\end{proof}

\subsection{Completing the proof of Lemma \ref{lemm:pred}}\label{sec:completingproofoflemmapred}

In this section, we complete the proof of Lemma \ref{lemm:pred} using Lemmas \ref{lemm:z(1)-z(2)}, \ref{lemm:g_t(1)-g_t(2)}, \ref{lemm:positivedefiniteVt} and \ref{lemm:generalself}.

By the definition of $\Delta(\x,\thetab)$, for all $\x\in\Dc$, $t\in[T]$ and $\thetab\in \Cc_t(\delta)$, we have the following chain of inequalities
\begin{align}
    \Delta(\x,\thetab) &= |\boldsymbol{\rho}^T\z(\x,\thetab_\ast)-\boldsymbol{\rho}^T\z(\x,\thetab)|\nn\\
    &\leq R\norm{\z(\x,\thetab_\ast)-\z(\x,\thetab)}_2 \tag{Assumption \ref{assum:boundedness} and Cauchy-Schwarz inequality} \\
    &= R\norm{\left[\B(\x,\thetab_\ast,\thetab)\otimes \x^T\right] (\thetab_\ast-\thetab)}_2 \tag{Lemma \ref{lemm:z(1)-z(2)}}\\
     &= R\norm{\left[\B(\x,\thetab_\ast,\thetab)\otimes \x^T\right]\Gb_t^{-1/2}(\thetab_\ast,\thetab) \Gb_t^{1/2}(\thetab_\ast,\thetab)(\thetab_\ast-\thetab)}_2 \nn\\
     &\leq R\norm{\left[\B(\x,\thetab_\ast,\thetab)\otimes \x^T\right]\Gb_t^{-1/2}(\thetab_\ast,\thetab)}_2 \norm{\thetab_\ast-\thetab}_{\Gb_t(\thetab_\ast,\thetab)}\tag{Cauchy-Schwarz inequality}\\
     &=R\sqrt{\lamax\left(\left[\B(\x,\thetab_\ast,\thetab)\otimes \x^T\right]\Gb_t^{-1}(\thetab_\ast,\thetab)\left[\B^T(\x,\thetab_\ast,\thetab)\otimes \x\right]\right)}\norm{\g_t(\thetab_\ast)-\g_t(\thetab)}_{\Gb_t^{-1}(\thetab_\ast,\thetab)}\tag{Lemma \ref{lemm:g_t(1)-g_t(2)}}\\
     &=R\sqrt{\lamax\left(\Gb_t^{-1/2}(\thetab_\ast,\thetab)\left[\B^T(\x,\thetab_\ast,\thetab)\otimes \x\right]\left[\B(\x,\thetab_\ast,\thetab)\otimes \x^T\right]\Gb_t^{-1/2}(\thetab_\ast,\thetab)\right)}\norm{\g_t(\thetab_\ast)-\g_t(\thetab)}_{\Gb_t^{-1}(\thetab_\ast,\thetab)}\tag{cyclic property of $\lamax$}\\
     &=R\sqrt{\lamax\left(\Gb_t^{-1/2}(\thetab_\ast,\thetab)\left[\B^T(\x,\thetab_\ast,\thetab)\B(\x,\thetab_\ast,\thetab)\otimes \x\x^T\right]\Gb_t^{-1/2}(\thetab_\ast,\thetab)\right)}\norm{\g_t(\thetab_\ast)-\g_t(\thetab)}_{\Gb_t^{-1}(\thetab_\ast,\thetab)}\tag{mixed-product property}\\
     &\leq RL\sqrt{\lamax\left(\Gb_t^{-1}(\thetab_\ast,\thetab)\left[I_K\otimes \x\x^T\right]\right)} \norm{\g_t(\thetab_\ast)-\g_t(\thetab)}_{\Gb_t^{-1}(\thetab_\ast,\thetab)}\tag{Assumption \ref{assum:problemdependentconstants}~~ (*)}\\
     &= RL\sqrt{\lamax\left(\Gb_t^{-1}(\thetab_\ast,\thetab)\left[I_K\otimes \x\right]\left[I_K\otimes\x^T\right]\right)}\norm{\g_t(\thetab_\ast)-\g_t(\thetab)}_{\Gb_t^{-1}(\thetab_\ast,\thetab)}\tag{mixed-product property}\\
       &= RL\sqrt{\lamax\left(\left[I_K\otimes \x^T\right]\Gb_t^{-1}(\thetab_\ast,\thetab)\left[I_K\otimes\x\right]\right)}\norm{\g_t(\thetab_\ast)-\g_t(\thetab)}_{\Gb_t^{-1}(\thetab_\ast,\thetab)}\tag{cyclic property of $\lamax$}\\
       &\leq RL\sqrt{\kappa\lamax\left(\left[I_K\otimes \x^T\right]\left[I_K\otimes \Vb_t^{-1}\right]\left[I_K\otimes\x\right]\right)}\norm{\g_t(\thetab_\ast)-\g_t(\thetab)}_{\Gb_t^{-1}(\thetab_\ast,\thetab)}\tag{Lemma \ref{lemm:positivedefiniteVt}~~(**)}\\
       &= RL\sqrt{\kappa\lamax\left(I_K\otimes\norm{\x}_{\Vb_t^{-1}}^2 \right)}\norm{\g_t(\thetab_\ast)-\g_t(\thetab)}_{\Gb_t^{-1}(\thetab_\ast,\thetab)}\tag{mixed-product property}\\
     &= RL\sqrt{ \kappa }\norm{\x}_{\Vb_t^{-1}}\norm{\g_t(\thetab_\ast)-\g_t(\hat\thetab_t)+\g_t(\hat\thetab_t)-\g_t(\thetab)}_{\Gb_t^{-1}(\thetab_\ast,\thetab)}\nn\\
     &\leq RL\sqrt{\kappa }\norm{\x}_{\Vb_t^{-1}} \left[\norm{\g_t(\thetab_\ast)-\g_t(\hat\thetab_t)}_{\Gb_t^{-1}(\thetab_\ast,\thetab)}+\norm{\g_t(\hat\thetab_t)-\g_t(\thetab)}_{\Gb_t^{-1}(\thetab_\ast,\thetab)}\right]\tag{$\left(I_K\otimes\Vb_t\right)^{-1}=I_K\otimes\Vb_t^{-1}$}\\
     &\leq RL\sqrt{\kappa \left(1+2S\right)}\norm{\x}_{\Vb_t^{-1}}\left[\norm{\g_t(\thetab_\ast)-\g_t(\hat\thetab_t)}_{\Hb_t^{-1}(\thetab_\ast)}+\norm{\g_t(\hat\thetab_t)-\g_t(\thetab)}_{\Hb_t^{-1}(\thetab)}\right]\tag{Lemma \ref{lemm:generalself}}\\
     &\leq 2RL\beta_t(\delta)\sqrt{\kappa  \left(1+2S\right)}\norm{\x}_{\Vb_t^{-1}}\tag{Theorem \ref{thm:confidenceset}},
\end{align}
as desired. Above, \emph{cyclic property of $\la_\max$} boils down to the fact that for two matrices $\mathbf{M}_1,\mathbf{M}_2\in\mathbb{R}^{d\times d}$ the eigenvalues of $\M_1\M_2$ are the same as the eigenvalues of $\M_2\M_1$ \footnote{
To see why this is true: let $\lambda,\vb$ be an eigenvalue-eigenvector pair of $\M_1\M_2$, i.e. $\M_1\M_2\vb=\la\vb$. Then, $(\M_2\M_1)\M_2\vb=\la\M_2\vb$. Thus, $\la$ is also an eigenvalue of $\M_2\M_1$.
}; thus, the same is true of the maximum eigenvalue.
For completeness, we also detail the derivation of the third in line inequality marked with (*) above. First, from an argument identical to the proof of Lemma \ref{lemm:positivedefiniteVt}, we have that $\B^T(\x,\thetab_\ast,\thetab)\B(\x,\thetab_\ast,\thetab)\otimes \x\x^T\preceq L^2 \big(I_K\otimes\x\x^T\big)$. From this and Theorem 7.7.2 in \cite{horn2012matrix} it follows that
$\Gb_t^{-1/2}(\thetab_\ast,\thetab)\big[\B^T(\x,\thetab_\ast,\thetab)\B(\x,\thetab_\ast,\thetab)\otimes \x\x^T\big] \Gb_t^{-1/2}(\thetab_\ast,\thetab) \preceq L^2 \Gb_t^{-1/2}(\thetab_\ast,\thetab)\big[I_K\otimes\x\x^T\big]\Gb_t^{-1/2}(\thetab_\ast,\thetab)$. The argument then follows directly by the applying the cyclic property of $\la_\max$.

\subsection{Completing the proof of Theorem \ref{thm:regretbound}}

Consider the standard instantaneous regret decomposition as follows:
\begin{align}
    r_t &= |\boldsymbol{\rho}^T\z(\x_\ast,\thetab_\ast)-\boldsymbol{\rho}^T\z(\x_t,\thetab_\ast)|\nn\\
    &\leq|\boldsymbol{\rho}^T\z(\x_\ast,\thetab_t)+\epsilon_t(\x_\ast)-\boldsymbol{\rho}^T\z(\x_t,\thetab_\ast)|\nn\\
    &\leq |\boldsymbol{\rho}^T\z(\x_t,\thetab_t)-\boldsymbol{\rho}^T\z(\x_t,\thetab_\ast)+\epsilon_t(\x_t)|\nn\\
    &\leq 2\epsilon_t(\x_t).
\end{align}
Therefore
\begin{align}
    R_T&\leq 2\sum_{t=1}^T\epsilon_t(\x_t)\nn\\
    &= 4RL\beta_T(\delta)\sqrt{\kappa \left(1+2S\right)}\sum_{t=1}^T \norm{\x_t}_{\Vb_t^{-1}}\tag{Corollary \ref{corr:theonly}}\\
    &\leq 4RL\beta_T(\delta)\sqrt{\kappa \left(1+2S\right)} \sqrt{T\sum_{t=1}^T\norm{\x_t}^2_{\Vb_t^{-1}}}\tag{Cauchy-Schwarz inequality}\\
    &\leq 4RL\beta_T(\delta)\sqrt{2\max(1,\frac{1}{\la\kappa})\kappa \left(1+2S\right)dT\log\left(1+\frac{T}{\kappa \la d}\right)}.
\end{align}

In the last inequality, we used the standard argument in regret analysis of linear bandit algorithm stated in the following \cite{abbasi2011improved} (Lemma~11):
\begin{align}\label{eq:standardarg}
    \sum_{t=1}^n \min\left(\norm{\y_t}^2_{\Vb_{t}^{-1}},1\right)\leq 2\log\frac{\det \A_{n+1}}{\det\A_{1}}\quad \text{where}\quad \A_n = \A_{1}+\sum_{t=1}^{n-1}\y_t\y^T_t.
\end{align}

\subsection{Discussion on possible improvements}
Our theorem proves a regret of order $\Oc(\sqrt{\kappa})$ with respect to the ``problematic" (since it can grow exponentially large) smoothness parameter $\kappa$ that we identified in Assumption \ref{assum:problemdependentconstants}. As we discussed this improves upon a ``naive" $\Oc(\kappa)$ rate that would result from the looser confidence set in \eqref{eq:badconfidenceset}. To achieve the improvement of Theorem \ref{thm:regretbound}, we extended the ideas of \cite{faury2020improved} to the multinomial case. As we saw throughout the proof, this required circumventing several technical intricacies that are unique in the multinomial case. Thanks to a careful treatment (e.g. in Section \ref{sec:completingproofoflemmapred}), we were also able to obtain the optimal dependency of the regret on $K$, i.e. $R_T=\Oc(K)$.

One might wonder whether it is possible to further improve the dependency of the regret to the problematic parameter $\kappa$. To answer this, it is reasonable to first consider the binary case $K=1$. For this, \cite{faury2020improved} showed that it is in fact possible, by developing a (more complicated) logistic-bandit algorithm with regret $\Oc(d\sqrt{T}\log(T)+\kappa d^2\log^2(T))$. Specifically, the dependency on $\kappa$ is pushed in a second order term that grows very slowly with the time horizon compared to the first term. Is it possible to extend this to the multinomial case?

A close inspection of the proof presented in the previous subsections reveals that the  ``extra'' $\sqrt{\kappa}$  factor enters the regret in Equation (**) in the display in Section \ref{sec:completingproofoflemmapred}. Specifically this follows when we replace the matrix $\Gb_t(\thetab_\ast,\thetab)$ with the simpler matrix $\tilde\Vb_t$ (cf. Lemma \ref{lemm:positivedefiniteVt}). In turn this is possible by replacing the key matrix $\A(\x,\thetab)$ ---the Jacobian of the MNL model--- by the smaller ---in the sense of the Loewner order--- matrix $\frac{1}{\kappa}I_K$ (cf. Assumption \ref{assum:boundedness}). In the binary case, the idea of \cite{faury2020improved} to circumvent this step is to introduce a refined ``local" lower bound to the derivative of the logistic model (corresponding to our Jacobian above). In the binary case, such a lower bound always exists. However, this is not at all clear in the multinomial case because the Loewner order is only a partial order. This is yet another demonstration that extensions to the multinomial case are challenging on their own right and might require careful treatment. Theorem \ref{thm:regretbound} circumvents such challenges showing a regret of $\Oc(\sqrt{\kappa})$. We leave answering the exciting question above on whether the dependency on $\kappa$ can be improved and whether this can be done while maintaining the optimal linear dependency on $K$ to future work.


\section{Miscellaneous Useful results}\label{sec:usefullemmas}

\subsection{A looser confidence set}\label{sec:looser_pf}
\begin{lemma}\label{lemm:badconfidenceset}
Let $\gamma_t(\delta):= 2\left(\sqrt{\la}S+2\sqrt{\log(1/\delta)+Kd\log\left(1+\frac{t}{\kappa \la d}\right)}\right)$. Then, it holds that $\thetab_\ast\in\Ec_t(\delta)$ with probability at least $1-\delta$.
\end{lemma}
\begin{proof}
\begin{align}
    \norm{\thetab_\ast-\tilde\thetab_t}_{\tilde\Vb_t} &\leq \sqrt{\kappa} \norm{\thetab_\ast-\tilde\thetab_t}_{\Gb_t(\thetab_\ast,\tilde\thetab_t)}\tag{Lemma \ref{lemm:positivedefiniteVt}}\\
    &= \sqrt{\kappa} \norm{\g_t(\thetab_\ast)-\g_t(\tilde\thetab_t)}_{\Gb^{-1}_t(\thetab_\ast,\tilde\thetab_t)}\tag{Lemma \ref{lemm:g_t(1)-g_t(2)original}}\\
    &\leq\sqrt{\kappa}\left(\norm{\g_t(\thetab_\ast)-\g_t(\hat\thetab_t)}_{\Gb^{-1}_t(\thetab_\ast,\tilde\thetab_t)}+\norm{\g_t(\hat\thetab_t)-\g_t(\tilde\thetab_t)}_{\Gb^{-1}_t(\thetab_\ast,\tilde\thetab_t)}\right)\nn\\
    &\leq\kappa\left(\norm{\g_t(\thetab_\ast)-\g_t(\hat\thetab_t)}_{\tilde\Vb^{-1}_t}+\norm{\g_t(\hat\thetab_t)-\g_t(\tilde\thetab_t)}_{\tilde\Vb_t^{-1}}\right)\tag{Lemma \ref{lemm:positivedefiniteVt}}\\
    &\leq 2\kappa\norm{\g_t(\thetab_\ast)-\g_t(\hat\thetab_t)}_{\tilde\Vb^{-1}_t}\tag{Definition of $\tilde\thetab_t$ and $\thetab_\ast\in\Theta$}\\
       &\leq 2\kappa\left(\sqrt{\la}S+\norm{\s_t}_{\tilde\Vb_t^{-1}}\right)\\
       &\leq 2\kappa\left(\sqrt{\la}S+2\sqrt{\log(1/\delta)+Kd\log\left(1+\frac{t}{\kappa \la d}\right)}\right)\tag{Theorem 1 in \cite{abbasi2011improved}}.
\end{align}
\end{proof}

\begin{lemma}\label{eq:tildeconfidenceset}
For all $t\in[T]$, it holds that $\Cc_t(\delta)\subseteq\tilde\Cc_t(\delta)$.
\end{lemma}
\begin{proof}

Let $\thetab\in\Cc_t(\delta)$.
\begin{align}
    \norm{\thetab-\thetab_t}_{\Hb_t(\thetab)} &\leq \sqrt{1+2S} \norm{\thetab-\thetab_t}_{\Gb_t(\thetab,\thetab_t)}\tag{Lemma \ref{lemm:generalself}}\\
    &= \sqrt{1+2S} \norm{\g_t(\thetab)-\g_t(\thetab_t)}_{\Gb^{-1}_t(\thetab,\thetab_t)}\tag{Lemma \ref{lemm:g_t(1)-g_t(2)}}\\
    &\leq\sqrt{1+2S}\left(\norm{\g_t(\thetab)-\g_t(\hat\thetab_t)}_{\Gb^{-1}_t(\thetab,\thetab_t)}+\norm{\g_t(\hat\thetab_t)-\g_t(\thetab_t)}_{\Gb^{-1}_t(\thetab,\thetab_t)}\right)\nn\\
    &\leq(1+2S)\left(\norm{\g_t(\thetab)-\g_t(\hat\thetab_t)}_{\Hb^{-1}_t(\thetab)}+\norm{\g_t(\hat\thetab_t)-\g_t(\thetab_t)}_{\Hb^{-1}_t(\thetab_t)}\right)\tag{Lemma \ref{lemm:generalself}}\\
    &\leq (2+4S)\beta_t(\delta)\tag{Theorem \ref{thm:confidenceset}}.
\end{align}
Thus, $\thetab\in\tilde\Cc_t(\delta)$ which completes the proof.
\end{proof}

\subsection{Proof of \texorpdfstring{ \eqref{eq:badexplorationbonus}}{Lg}}\label{sec:looser_pf_2}

Based on this confidence set and under the event $\{\boldsymbol\theta_\ast\in \Ec_{t}(\delta),~\forall t\in[T]\}$, we derive a new upper bound on $\Delta(\x,\thetab)$ for all $\x\in\Dc$ and $\thetab\in\Ec_t(\delta)$ as follows:
\begin{align}
    \Delta(\x,\thetab) &= |\boldsymbol{\rho}^T\z(\x,\thetab_\ast)-\boldsymbol{\rho}^T\z(\x,\thetab)|\nn\\
    &\leq R\norm{\z(\x,\thetab_\ast)-\z(\x,\thetab)}_2 \tag{Assumption \ref{assum:boundedness} and Cauchy-Schwarz inequality} \\
    &= R\norm{\left[\B(\x,\thetab_\ast,\thetab)\otimes \x^T\right] (\thetab_\ast-\thetab)}_2 \tag{Lemma \ref{lemm:z(1)-z(2)}}\\
     &= R\norm{\left[\B(\x,\thetab_\ast,\thetab)\otimes \x^T\right]\tilde\Vb_t^{-1/2} \tilde\Vb_t^{1/2}(\thetab_\ast-\thetab)}_2 \nn\\
     &\leq R\norm{\left[\B(\x,\thetab_\ast,\thetab)\otimes \x^T\right]\tilde\Vb_t^{-1/2}}_2 \norm{\thetab_\ast-\thetab}_{\tilde\Vb_t}\tag{Cauchy-Schwarz inequality}\\
     &= R\sqrt{\lamax\left(\left[\B(\x,\thetab_\ast,\thetab)\otimes \x^T\right]\tilde\Vb_t^{-1}\left[\B^T(\x,\thetab_\ast,\thetab)\otimes \x\right]\right)} \norm{\thetab_\ast-\thetab}_{\tilde\Vb_t}\nn\\
     &\leq RL\sqrt{\lamax\left(\left[I_K\otimes \Vb_t^{-1}\right]\left[I_K\otimes \x\x^T\right]\right)} \norm{\thetab_\ast-\thetab}_{\tilde\Vb_t}\tag{Assumption \ref{assum:problemdependentconstants} and cyclic property of $\lamax$}\\
     &= RL\sqrt{\lamax\left(\left[I_K\otimes \Vb_t^{-1}\right]\left[I_K\otimes \x\right]\left[I_K\otimes\x^T\right]\right)} \norm{\thetab_\ast-\thetab}_{\tilde\Vb_t}\tag{mixed-product property}\\
     &= RL\sqrt{\lamax\left(\left[I_K\otimes\x^T\right]\left[I_K\otimes \Vb_t^{-1}\right]\left[I_K\otimes \x\right]\right)} \norm{\thetab_\ast-\thetab}_{\tilde\Vb_t}\tag{cyclic property of $\lamax$}\\
     &= RL\norm{\x}_{\Vb_t^{-1}} \left[\norm{\thetab_\ast-\tilde\thetab_t}_{\tilde\Vb_t}+\norm{\tilde\thetab_t-\thetab}_{\tilde\Vb_t}\right]\nn\\
     &\leq 2RL\kappa\gamma_t(\delta)\norm{\x}_{\Vb_t^{-1}}\tag{Lemma \ref{lemm:badconfidenceset}}.
\end{align}

\subsection{A useful diagonally-dominant \texorpdfstring{$M$}{Lg}-matrix}

\begin{lemma}\label{lemm:diagonallyappendix}
For any $\x\in\mathbb{R}^d$ and $\thetab\in\mathbb{R}^{Kd}$, matrix  $\A(\x,\thetab)$ defined in \eqref{eq:A} is a strictly diagonally dominant $M$-matrix.
\end{lemma}
\begin{proof}
Recall that $\A(\x,\thetab)_{ij}:=z_i(\x,\thetab)\left(\delta_{ij}-z_j(\x,\thetab)\right)$ for all $i,j\in[K]$. Thanks to the MNL probabilistic model, one can easily observe that for each $i\in[K]$ of matrix $\A(\x,\thetab)$, we have:
\begin{align}
    \A(\x,\thetab)_{ii} = z_0(\x,\thetab)z_i(\x,\thetab)+\sum_{j\neq i}|\A(\x,\thetab)_{ij}|>\sum_{j\neq i}|\A(\x,\thetab)_{ij}|,\\
    \A(\x,\thetab)_{ii} = z_0(\x,\thetab)z_i(\x,\thetab)+\sum_{j\neq i}|\A(\x,\thetab)_{ji}|>\sum_{j\neq i}|\A(\x,\thetab)_{ji}|.
\end{align}
Therefore, for any $\x\in\mathbb{R}^d$ and $\thetab\in\mathbb{R}^{Kd}$, matrix $\A(\x,\thetab)$ is strictly diagonally dominant. Furthermore, a strictly diagonally dominant matrix which is also \emph{symmetric} with positive diagonal entries is positive definite (Theorem 6.1.10 in \cite{horn2012matrix}). Thus, $\A(\x,\thetab)$ is a positive definite matrix with positive eigenvalues, which paired with  the fact that all its off-diagonal entries are negative proves that it is also an $M$-matrix.
\end{proof}

\subsection{Proof of \texorpdfstring{ \eqref{eq:kappa_bounds}}{Lg}}\label{sec:kappa_bounds}

For convenience, define $\tilde{S}=S/\sqrt{K}$.

An application of Theorem~1.1 in \cite{tian2010inequalities} that provides upper and lower bounds on the minimum eigenvalue of a strictly diagonally dominant $M$-matrix, together with Lemma \ref{lemm:diagonally} give the following for all $\x$ and $\thetab$: 
  $  \min_{i\in[K]}\sum_{j=1}^K\A(\x,\thetab)_{ij}\leq\lamin(\A(\x,\thetab))\leq \max_{i\in[K]}\sum_{j=1}^K\A(\x,\thetab)_{ij}.
  $
 By the definition of $\kappa$ in Assumption \ref{assum:problemdependentconstants}, the previous lower bound and the definitions of $\A(\x,\thetab)$ and of the multinomial logit probability model we have then that
\begin{align}\label{eq:lowerboundonA}
\frac{1}{\kappa}&:=\min_{\x\in\Dc,\thetab\in\Theta}\lamin(\A(\x,\thetab))\geq \min_{\x\in\Dc,\thetab\in\Theta,i\in[K]}\sum_{j=1}^K\A(\x,\thetab)_{ij}\notag
\\
&= \min_{\x\in\Dc,\thetab\in\Theta,i\in[K]}\frac{e^{\bar\thetab_i^T\x}}{\left(1+\sum_{j=1}^Ke^{\bar\thetab_j^T\x}\right)^2} = \min_{i\in[n]}\min_{\x\in\Dc,\thetab\in\Theta}\frac{e^{\bar\thetab_i^T\x}}{\left(1+\sum_{j=1}^Ke^{\bar\thetab_j^T\x}\right)^2} \nn\\
&\geq
   \min_{i\in[n]}\frac{e^{-SX}}{(1+Ke^{SX})^2} = \frac{e^{-SX}}{(1+Ke^{SX})^2}.
\end{align}
The lower bound in the last line follows since for any $\x\in\Dc$ and $\thetab\in\Theta$ by Cauchy-Schwarz: $-SX \leq \thetab_i^T\x\leq SX$, for all $i\in[n]$. 
This proves the advertised upper bound on $\kappa$.  

For the lower bound, we proceed similarly to find that

\begin{align}\label{eq:upperboundonA}
   \frac{1}{\kappa}&:=\min_{\x\in\Dc,\thetab\in\Theta}\lamin(\A(\x,\thetab))
   \leq \min_{\x\in\Dc,\thetab\in\Theta}\max_{i\in[K]}\sum_{j=1}^K\A(\x,\thetab)_{ij}\notag\\
   &=\min_{\x\in\Dc,\thetab\in\Theta}\max_{i\in[K]}\frac{e^{\bar\thetab_i^T\x}}{\left(1+\sum_{j=1}^Ke^{\bar\thetab_j^T\x}\right)^2}\nn\\
   &\leq \max_{i\in[K]}\frac{e^{-\tilde{S}X}}{(1+Ke^{-\tilde{S}X})^2} = \frac{e^{-\tilde{S}X}}{(1+Ke^{-\tilde{S}X})^2}.\nn
   \end{align}
 The inequality in the last line follows by choosing  feasible $\x$ and $\thetab$ as follows. Let $\bar\thetab_1=\ldots=\bar\thetab_K=\bar\thetab$ with $\|\bar\thetab\|_2=\tilde{S}$ and $\x=-\frac{X}{\tilde{S}}\bar\thetab$.
 The above gives the desired upper bound and concludes the proof.

\end{document}

%% file: commands.tex
\newcommand{\lamin}{\la_{\rm \min}}
\newcommand{\lamax}{\la_{\rm \max}}
\newcommand{\la}{\lambda}

\newcommand{\nn}{\nonumber}

\newcommand{\bal}{\begin{align}}
\newcommand{\eal}{\end{align}}

\DeclarePairedDelimiterX{\inp}[2]{\langle}{\rangle}{#1, #2}

\newcommand{\etab}{\boldsymbol{\eta}}

\newcommand{\M}{\mathbf{M}}
\newcommand{\Gb}{\mathbf{G}}
\newcommand{\Hb}{\mathbf{H}}
\newcommand{\A}{\mathbf{A}}
\newcommand{\B}{\mathbf{B}}

\newcommand{\Vb}{\mathbf{V}}

\newcommand{\x}{\mathbf{x}}

\newcommand{\g}{\mathbf{g}}
\newcommand{\vb}{\mathbf{v}}

\newcommand{\e}{\mathbf{e}}
\newcommand{\y}{\mathbf{y}}
\newcommand{\s}{\mathbf{s}}
\newcommand{\z}{\mathbf{z}}

\newcommand{\f}{\mathbf{f}}

\newcommand{\Fc}{\mathcal{F}}

\newcommand{\Sc}{{\mathcal{S}}}
\newcommand{\Bc}{{\mathcal{B}}}

\newcommand{\Dc}{\mathcal{D}}

\newcommand{\Lc}{\mathcal{L}}
\newcommand{\Cc}{\mathcal{C}}

\newcommand{\Ec}{\mathcal{E}}

\newcommand{\Oc}{\mathcal{O}}

\newcommand{\beq}{\begin{equation}}
\newcommand{\eeq}{\end{equation}}
\newcommand{\bea}{\begin{align}}
\newcommand{\eea}{\end{align}}

\newcommand{\Otilde}{\tilde\Oc}

\newcommand{\thetab}{\boldsymbol\theta}
\newcommand{\epsilonb}{\boldsymbol\epsilon}
\newcommand{\nub}{\boldsymbol\nu}

\newcommand{\xib}{\boldsymbol\xi}
\newcommand{\q}{\mathbf{q}}
\newcommand{\m}{\mathbf{m}}
\newcommand{\diag}{\rm{diag}}